\documentclass{article}

\usepackage{arxiv}

\usepackage{pifont}%
\usepackage[T1]{fontenc}    %
\usepackage{hyperref}       %
\usepackage{url}            %
\usepackage{booktabs}       %
\usepackage{amsfonts}       %
\usepackage{nicefrac}       %
\usepackage{microtype}      %
\usepackage{lipsum}
\usepackage{graphicx}
\usepackage{dsfont}
\usepackage{xcolor}
\usepackage{wrapfig}
\usepackage{amssymb}
\usepackage{natbib}
\usepackage{subcaption}
\usepackage{tcolorbox}
\usepackage{enumitem}
\usepackage{xcolor}

\usepackage{tikz}
\newcommand{\tikzxmark}{%
\tikz[scale=0.23] {
    \draw[line width=0.7,line cap=round] (0,0) to [bend left=6] (1,1);
    \draw[line width=0.7,line cap=round] (0.2,0.95) to [bend right=3] (0.8,0.05);
}}

\graphicspath{ {./images/} }

\usepackage{amsmath,amsfonts,bm}

\def\eqref#1{equation~\ref{#1}}

\def\1{\bm{1}}

\DeclareMathAlphabet{\mathsfit}{\encodingdefault}{\sfdefault}{m}{sl}
\SetMathAlphabet{\mathsfit}{bold}{\encodingdefault}{\sfdefault}{bx}{n}

\DeclareMathOperator*{\argmin}{arg\,min}

\newtheorem{theorem}{Theorem}
\newtheorem{proof}{Proof}
\newtheorem{remark}{Remark}

\title{
  SMILE: Zero-Shot Sparse Mixture of Low-Rank Experts Construction From Pre-Trained Foundation Models
}
\renewcommand{\shorttitle}{SMILE: Zero-Shot Sparse Mixture of Low-Rank Experts Construction}

\author{
    Anke Tang\textsuperscript{1}, 
    Li Shen\textsuperscript{2}, 
    Yong Luo\textsuperscript{1}, 
    Shuai Xie\textsuperscript{3}, 
    Han Hu\textsuperscript{4}, 
    Lefei Zhang\textsuperscript{1}, 
    Bo Du\textsuperscript{1}, 
    Dacheng Tao\textsuperscript{5} \\
    \textsuperscript{1}Wuhan University, \textsuperscript{2}Sun Yat-sen University, \textsuperscript{3}JD Explore Academy,\\
    \textsuperscript{4}Beijing Institute of Technology,
    \textsuperscript{5}Nanyang Technological University \\
    \textsuperscript{1}\texttt{\{anketang,luoyong,zhanglefei,dubo\}@whu.edu.cn},
    \textsuperscript{2}\texttt{mathshenli@gmail.com},\\
    \textsuperscript{3}\texttt{xieshuai@jd.com},
    \textsuperscript{4}\texttt{hhu@bit.edu.cn},
    \textsuperscript{5}\texttt{dacheng.tao@ntu.edu.sg}
}
\date{}

\begin{document}
\maketitle

\vspace{-20pt}
\begin{abstract}
  Deep model training on extensive datasets is increasingly becoming cost-prohibitive, prompting the widespread adoption of deep model fusion techniques to leverage knowledge from pre-existing models.
  From simple weight averaging to more sophisticated methods like AdaMerging, model fusion effectively improves model performance and accelerates the development of new models.
  However, potential interference between parameters of individual models and the lack of interpretability in the fusion progress remain significant challenges.
  Existing methods often try to resolve the parameter interference issue by evaluating attributes of parameters, such as their magnitude or sign, or by parameter pruning.
  In this study, we begin by examining the fine-tuning of linear layers through the lens of subspace analysis and explicitly define parameter interference as an optimization problem to shed light on this subject.
  Subsequently, we introduce an innovative approach to model fusion called zero-shot Sparse MIxture of Low-rank Experts (SMILE) construction, which allows for the upscaling of source models into an MoE model without extra data or further training.
  Our approach relies on the observation that fine-tuning mostly keeps the important parts from the pre-training, but it uses less significant or unused areas to adapt to new tasks. Also, the issue of parameter interference, which is intrinsically intractable in the original parameter space, can be managed by expanding the dimensions.
  We conduct extensive experiments across diverse scenarios, such as image classification and text generation tasks, using full fine-tuning and LoRA fine-tuning, and we apply our method to large language models (CLIP models, Flan-T5 models, and Mistral-7B models), highlighting the adaptability and scalability of SMILE.
  For full fine-tuned models, about 50\% additional parameters can achieve around 98-99\% of the performance of eight individual fine-tuned ViT models, while for LoRA fine-tuned Flan-T5 models, maintaining 99\% performance with only 2\% extra parameters.
  Code is available at \url{https://github.com/tanganke/fusion_bench}.
\end{abstract}


\section{Introduction}
\label{sec:introduction}

In recent years, the field of deep learning has witnessed an exponential growth in model sizes and dataset scales, making the training of large-scale deep models on extensive datasets increasingly cost-prohibitive, both in terms of financial resources and environmental impact~\citep{minaeeLargeLanguageModels2024,hadiLargeLanguageModels2024}.
Deep model fusion techniques have emerged as a promising solution, allowing the integration of knowledge from pre-existing models without the need for extensive retraining~\citep{liDeepModelFusion2023,zhengLearnModelFineTuning2023,yangModelMergingLLMs2024}.
This approach not only reduces computational costs but also enables the creation of more robust and versatile models by combining the strengths of multiple models.

Following the categorization in~\citet{tangFusionBenchComprehensiveBenchmark2024}, we classify model fusion methods into three main categories: model ensemble methods, model merging methods, and model mixing methods.
Model ensemble techniques aggregate the predictions from several models to enhance performance~\citep{sagiEnsembleLearningSurvey2018}. 
While resource-intensive in terms of memory and computation, it improves knowledge distillation training~\citep{wanKnowledgeFusionLarge2024,wanFuseChatKnowledgeFusion2024}.
Model merging methods, on the other hand, combine the parameters of multiple models into a single model, often through weighted averaging or parameter alignment~\citep{michaelmatenaMergingModelsFisherWeighted2022,xisenjinDatalessKnowledgeFusion2023}.
Model mixing methods involve the integration of multiple models through gating mechanisms or depth concatenation, allowing for more flexible and adaptive fusion strategies~\citep{komatsuzakiSparseUpcyclingTraining2023,kimSOLAR107B2023}.
These methods are particularly effective in multi-task learning scenarios, where the merged model can simultaneously perform multiple tasks.

However, despite the promising advancements in model fusion, several critical challenges persist, hindering the full realization of its potential.
A primary concern is the potential interference between parameters of different models, which leads to suboptimal performance.
Additionally, the lack of interpretability in the fusion process remains a significant hurdle, as current insights are largely confined to heuristic observations or simplified assumptions, such as linear mode connectivity, parameter signs or importance~\citep{ainsworthGitReBasinMerging2023,georgestoicaZipItMergingModels2023,yadavResolvingInterferenceWhen2023,yuLanguageModelsAre2023}.
Understanding how parameters are merged is crucial for building trust in the merged models and for further improving fusion techniques.
These challenges are particularly pronounced in complex, high-dimensional, non-linear model architectures, where the interactions between parameters can be extremely intricate and non-intuitive.

\begin{wrapfigure}{r}{0.4\textwidth}
  \centering
  \vspace{-10pt}
  \includegraphics[width=\linewidth]{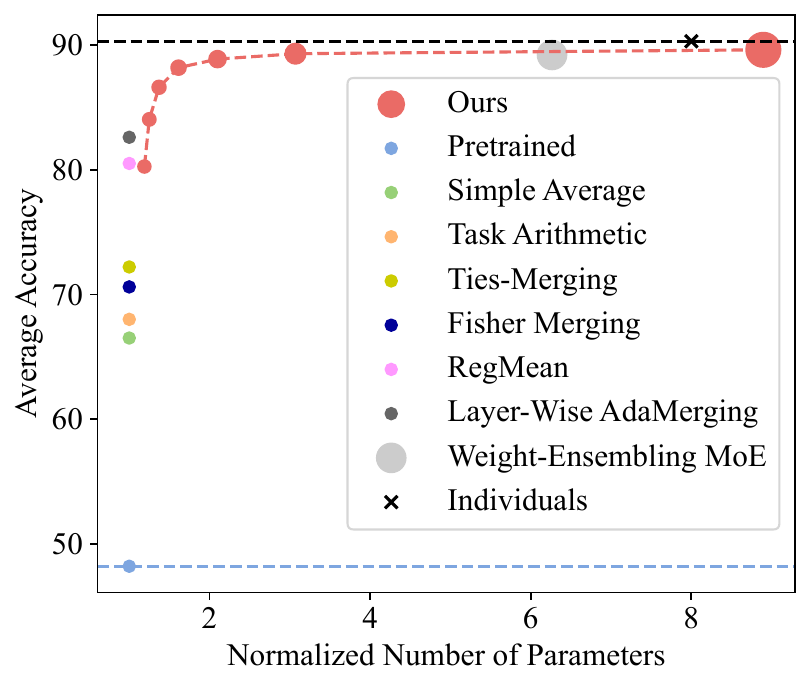}
  \caption{
    Multi-task model fusion experiment on eight image classification tasks using CLIP-ViT-B/32 models.
    Here we set $k_{gate}=16$ and $k$ is varied from 4 to 128 to investigate the trade-off between performance and model size.
  }
  \label{fig:clip-vit-b-32_scatter}
  \vspace{-15pt}
\end{wrapfigure}

Instead of relying on heuristic methods or simplified assumptions, we propose a novel subspace perspective on understanding and addressing the parameter interference problem in this study.
We first examine the fine-tuning process in linear layers through the lens of subspace analysis using matrix decomposition in Section~\ref{sec:revisiting_fine_tuning}.
This allows us to decompose the prediction of a fine-tuned model into distinct components, encompassing the pre-trained knowledge and task-specific adaptation. This approach provides insights into how models adapt to downstream tasks while preserving pre-trained knowledge.
Drawing from experimental observations, we build a more comprehensive understanding of fine-tuning, we further formulate parameter interference as an optimization problem in Section~\ref{sec:parameter_interference}, providing a more rigorous and measurable perspective.

Based on our insights, we introduce an innovative approach called zero-shot Sparse MIxture of Low-rank Experts (SMILE) construction, enhancing existing source models into a more versatile MoE model.
The zero-shot aspect of our approach is particularly noteworthy, as it facilitates the immediate deployment of fused models in new environments or tasks, drastically minimizing the time and resources typically required for model adaptation.

The effectiveness of our proposed method is rooted in two key observations derived from our subspace analysis.
Firstly, we found that the fine-tuning largely preserves the most important pre-trained weights and primarily utilizes less significant or previously unused dimensions of the parameter space to adapt to new tasks.
This preservation ensures that the critical pre-training knowledge encoded in the original models is not lost during fine-tuning and implies that the parameter subspace required to accommodate new knowledge may vary from task to task.
Secondly, we found that while parameter interference is inherently difficult to address in the original parameter space, it becomes more manageable when we increase the model's dimensionality. This expansion creates additional `room' for task-specific parameter updates to coexist without mutual interference.

We conducted extensive experiments across various tasks and models in both the vision and language domains, utilizing traditional full fine-tuning as well as Low-Rank Adaptation (LoRA)~\citep{huLoRALowRankAdaptation2021}. The results show that for models that undergo full fine-tuning, adding approximately 50\% more parameters allows us to achieve around 98-99\% of the performance of eight individual fine-tuned models. In the case of LoRA fine-tuned models, maintaining 99\% of the individual performance requires only a 2\% increase in parameters. This method also offers trade-offs between performance and model size, as illustrated in Figure~\ref{fig:clip-vit-b-32_scatter}, where we vary the rank $k$ of local experts.

To summarize, our contributions in this study are as follows:
\begin{itemize}[noitemsep,nolistsep]
  \item We provide a novel subspace perspective on the fine-tuning process, shedding light on how models adapt to new tasks while preserving pre-trained knowledge. In addition, We formulate the parameter interference problem as an optimization problem, providing a more rigorous and measurable perspective on this issue.
  \item We introduce a zero-shot Sparse Mixture of Low-Rank Experts (SMILE) construction approach, enabling the fusion of existing models into more unified versatile SMILE models. We also discuss the complexity of our method, highlighting its potential for broader applications in deep learning research and practice.
  \item We demonstrate the effectiveness of our method through extensive experiments on a variety of tasks and setups, showcasing its superior performance and efficiency compared to existing model fusion techniques.
\end{itemize}

\section{Rethinking Model Fine-Tuning From a Subspace Perspective}
\label{sec:revisiting_fine_tuning}

In this study, we aim to construct a unified versatile model from multiple fine-tuned models, which can perform multiple tasks or handle inputs from multiple domains simultaneously. We denote the number of fine-tuned models as $T$.
Before we delve into the proposed method's details, we gain insights into the fine-tuning process from a singular value decomposition (SVD) subspace perspective.
In this section, we aim to (1) investigate and locate the task information in the fine-tuned weights $W_{ft}$, and (2) understand how it is related to the pre-trained weights $W$.

Consider a linear layer of the pre-trained model with weight matrix $W \in \mathbb{R}^{m \times n}$ and bias vector $b \in \mathbb{R}^{m}$.
After full fine-tuning on a downstream task, the weight matrix and bias vector are updated to $W_{ft}$ and $b_{ft}$, respectively.
To achieve a deeper understanding of these updates, we need to employ mathematical tools that allow us to decompose the parameter space into distinct ranges of importance, i.e. subspaces.
We state the following theorem.
\begin{theorem}
  \label{thm:matrix_subspace}
  Given two sets of orthonormal vectors $\{u_i\}_{i=1}^{p}\subset \mathbb{R}^m$ and $\{v_i\}_{i=1}^{q}\subset \mathbb{R}^n$, $1 \leq p\leq m$ and $1\leq q\leq n$, the set of matrices $\{u_i v_j^T\}_{i=1,j=1}^{p,q}$ forms an orthonormal basis for a subspace of $\mathbb{R}^{m \times n}$ with dimension $pq$.
\end{theorem}
\begin{proof}
  For simplicity, let $x_{ij} = u_i v_j^T$.
  The Frobenius inner product of two matrices $x_{ab}$ and $x_{cd}$ is defined as
  \begin{equation}
    \langle x_{ab}, x_{cd} \rangle = \operatorname{tr} \left(u_a v_b^T(u_c v_d^T)^T \right) = \operatorname{tr} \left(u_a v_b^T v_d u_c^T\right) \in \mathbb{R}.
  \end{equation}

  \textbf{Orthonormality}: we consider three cases:
  \begin{enumerate}
    \item If $a = c$ and $b = d$, then $\langle x_{ab}, x_{cd} \rangle = \operatorname{tr} \left( u_a u_a^T \right) = u_a^\top u_a = 1$.
    \item If $b \neq d$, then $v_b^T v_d = 0$ and $\langle x_{ab}, x_{cd} \rangle = 0$.
    \item If $b = d$ and $a \neq c$, then $v_b^T v_d = 1$ and $\langle x_{ab}, x_{cd} \rangle =  \operatorname{tr} (u_a u_c^\top) = u_a^\top u_c = 0$.
  \end{enumerate}
  Thus, $\{x_{ij}\}_{i,j=1}^{p,q}$ is orthonormal.

  \textbf{Linear Independence}:
  assume there exists a nonzero matrix $\alpha \in \mathbb{R}^{p\times q}$ such that $\sum_{i=1}^{p} \sum_{j=1}^{q} \alpha_{ij} x_{ij} = 0$. For any $a\in[p]$and $b\in[q]$, take the inner product of both sides with $x_{ab}$. We obtain the following:
  \begin{equation}
    \left\langle \sum_{i=1}^{p} \sum_{j=1}^{q} \alpha_{ij} x_{ij}, x_{ab} \right\rangle = \langle 0, x_{ab} \rangle = 0.
  \end{equation}
  By the linearity of the inner product and the orthogonality, we proved:
  \begin{equation}
    \sum_{i=1}^{p} \sum_{j=1}^{q} \alpha_{ij} \left\langle x_{ij}, x_{ab} \right\rangle = \alpha_{ab} \left\langle x_{ab}, x_{ab} \right\rangle = \alpha_{ab} = 0.
  \end{equation}
  Since this holds for any $a$ and $b$, we conclude that all $\alpha_{ij} = 0$.
  This leads to a contradiction to the assumption that $\alpha$ is nonzero.
  Therefore, the set $\{x_{ij}\}_{i,j=1}^{r}$ is linearly independent, which is the necessary and sufficient conditions for a set of elements to form a basis for a vector space with dimension $pq$.
\end{proof}

\begin{figure}[t]
  \centering
  \begin{subfigure}[t]{0.42\textwidth}
    \centering
    \includegraphics[width=\textwidth]{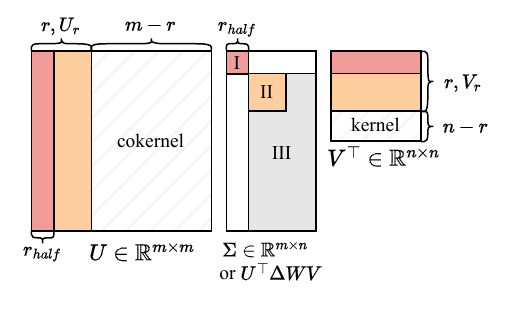}
    \caption{Full SVD and subspace partition.}
    \label{fig:svd}
  \end{subfigure}
  \hfill
  \begin{subfigure}[t]{0.57\textwidth}
    \centering
    \includegraphics[width=\linewidth]{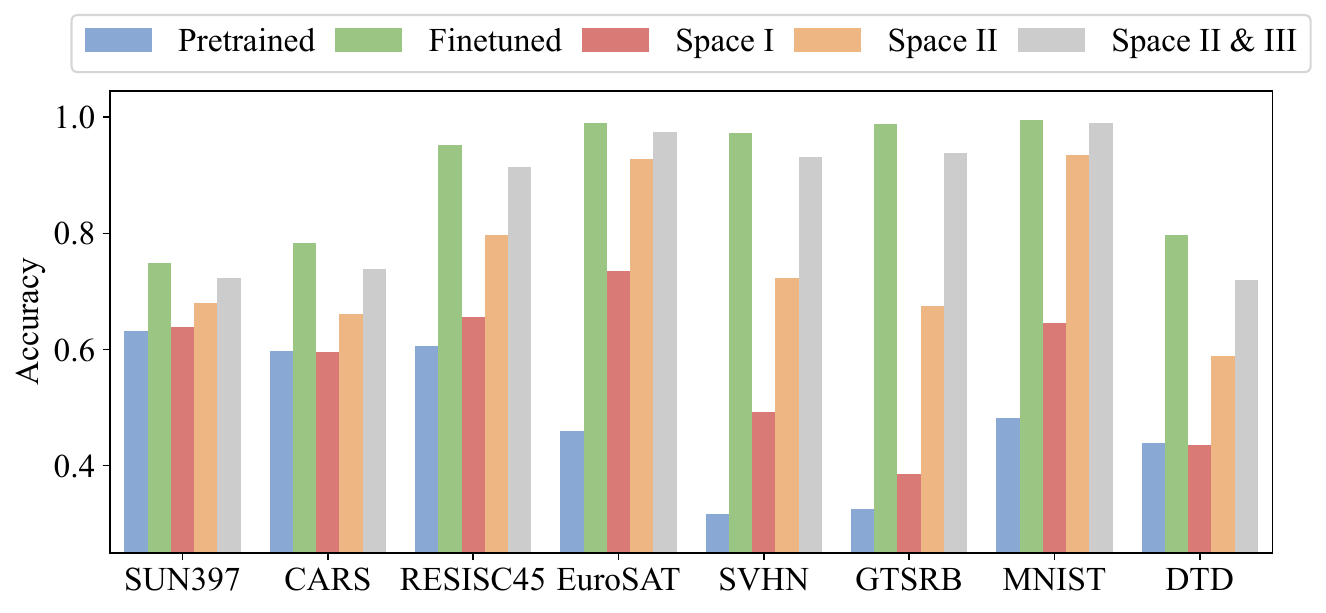}
    \caption{Accuracy comparison across different subspace projection strategies.}
    \label{fig:single-task}
  \end{subfigure}
  \caption{
    Here we show the SVD decomposition and subspace partition of the singular value matrix $\Sigma$, and the accuracy comparison of different subspace projection strategies discussed in Section~\ref{sec:revisiting_fine_tuning}.
  }
\end{figure}
We start by decomposing the weight matrix $W$ using the reduced SVD as $W = U_r \Sigma_r V_r^T$, where $U_r \in \mathbb{R}^{m \times r}$ and $V_r \in \mathbb{R}^{r \times n}$ are orthonormal matrices containing the left singular vectors and right singular vectors, respectively, $\Sigma_r \in \mathbb{R}^{r \times r}$ is a diagonal matrix containing the singular values sorted in descending order, and $r$ is the rank of the matrix $W$~\citep{olverAppliedLinearAlgebra2018}.
In the case of full SVD, the matrices are $U \in \mathbb{R}^{m \times m}$, $\Sigma \in \mathbb{R}^{m \times n}$, and $V \in \mathbb{R}^{n \times n}$, which preserve all information about the matrix $W$, including its kernel (null space) and cokernel (left null space), as shown in Figure~\ref{fig:svd}.
\begin{remark}
  According to Theorem~\ref{thm:matrix_subspace}, the set of matrices $\{u_i v_j^T | i\in[m], j\in[n]\}$ forms an orthonormal basis for a subspace of $\mathbb{R}^{m \times n}$ with dimension $mn$.
  In other words, for any real matrix $A\in\mathbb{R}^{m\times n}$, we can express it as a weighted sum of the elements in the basis, i.e. $A = \sum_{i=1}^{m} \sum_{j=1}^{n} \langle A, u_i v_j^T \rangle u_i v_j^T \in \operatorname{span}(\{u_i v_j^T\}_{i,j}^{m,n})$.
\end{remark}
\begin{remark}
  Let $\mathcal{U}$ and $\mathcal{V}$ be two subsets of $\{u_i\}_{i=1}^{m}$ and $\{v_i\}_{i=1}^{n}$, respectively. $\{u v^T| u\in \mathcal{U}, v\in \mathcal{V}\}$ forms a orthonormal basis for a subspace of $\mathbb{R}^{m\times n}$, with dimension $|\mathcal{U}||\mathcal{V}|\leq mn$.
\end{remark}
\label{remark:matrix_decomposition}
To gain insights into how fine-tuning modifies the pre-trained weights to adapt them to a specific task, we assume the fine-tuned linear layer accepts an input $x \in \mathbb{R}^{n}$ and outputs $y = W_{ft} x + b_{ft}$. Because the row space $\{v_i\}_{i=1}^{n}$ is an orthonormal basis for $\mathbb{R}^n$, we can decompose $x$ as $x = \sum_{i=1}^{n} \langle x, v_i \rangle v_i$, where $\langle \cdot, \cdot \rangle$ denotes the vector inner product. On the other hand, $W_{ft}$ and $b_{ft}$ are updated from the pre-trained weights $W$ and $b$. We can derive the following equation:
\begin{equation}
  \label{eq:ft_decomposition}
  y  = W_{ft} x + b_{ft}
  = (W + \Delta W) x + b + \Delta b
  = \underbrace{W x + b}_{\text{pre-trained part}} +
  \underbrace{\Delta W x + \Delta b}_{\text{fine-tuned part}}.
\end{equation}
Now we expand the pre-trained part and fine-tuned part in Eq.(\ref{eq:ft_decomposition}) separately as follows:
\begin{align}
  \text{pre-trained part} & = \sum_{i=1}^{n} W  \langle x, v_i \rangle v_i + b
  = \sum_{i=1}^{n} \left(\sum_{j=1}^{r} \sigma_j u_j v_j^\top\right) \langle x, v_i \rangle v_i + b                                        \\
                          & = \sum_{i=1}^{n} \sum_{j=1}^{r} \sigma_j \langle x, v_i \rangle u_j v_j^\top v_i + b
  = \sum_{j=1}^{r} \sigma_j \langle x, v_j \rangle u_j + b, %
  \\
  \text{fine-tuned part}  & = \sum_{i=1}^{n}  (W_{ft} - W) \langle x, v_i \rangle v_i + \Delta b
  = \sum_{i=1}^{n} \left(\sum_{j=1}^{m} \sum_{k=1}^{n} \delta_{jk} u_j v_k^\top\right) \langle x, v_i \rangle v_i  + \Delta b              \\
  \label{eq:ft_decomposition_ft_part}
                          & = \sum_{i=1}^{n} \sum_{j=1}^{m} \sum_{k=1}^{n} \delta_{jk} \langle x, v_i \rangle u_j v_k^\top v_i  + \Delta b
  = \sum_{j=1}^{m} \sum_{k=1}^{n} \delta_{jk} \langle x, v_k \rangle u_j +  \Delta b. %
\end{align}
Where $\delta_{jk} = \langle \Delta W, u_j v_k^\top \rangle = (U^\top \Delta W V)_{jk}$ is the Frobenius inner product between the fine-tuned weight update $\Delta W$ and the rank-one matrix $u_j v_k^\top$.
It also quantifies how much the weight updates align with the direction specified by $u_j v_k^\top$ and indicates which input-output transformation is enhanced or suppressed (or enhanced reversely) during fine-tuning, based on its sign and magnitude.
For example, a large positive $\delta_{jk}$ suggests that the connection between the $k$-th input direction ($v_k$) and $j$-th output direction ($u_k$) is strengthened for the downstream task.
This decomposition shows how different the pre-trained and fine-tuned parts contribute to the output. So far, we only understand that the fine-tuned update $\Delta W$ potentially uses all $mn$ dimensions, while the pre-trained part only uses $r$ dimensions.

We further split left/right singular vectors into three distinct subsets based on the distribution of the singular values of $\Delta W$, and design an ablation study corresponding to different zones in the projection coefficient matrix $U^\top \Delta W V$:
(1) The top-left zone I contains the most significant singular values that cumulatively account for 50\% of the total sum of the singular values, we denote the number of singular values in this zone as $r_{half}$, $\sum_{i=1}^{r_{half}} \sigma_i \approx  \sum_{i=1}^{r} \sigma_i/2$.
This zone is crucial for preserving pre-training information.
(2) The middle zone II encompasses the singular values that make up the remaining half of the cumulative sum. These values are still important but less so than those in zone I.
(3) Zone III contains no information about the pre-trained weights, as its range is beyond $\operatorname{rank}(W)$.
This zone partition is illustrated as the $\Sigma$ in Figure~\ref{fig:svd}.
We fine-tune the pre-trained CLIP-ViT-B/32 model on eight downstream tasks from different domains, including hand-written digit images, satellite images, regular patterns, car images, and natural images. Then we project the $\Delta W$ onto the different subspaces, including
(1) subspace I $\operatorname{span}(\{u_i v_j^T| i\in[r_{half}], j\in[r_{half}]\})$,
(2) subspace II $\operatorname{span}(\{u_i v_j^T| i\in[r_{half}, r], j\in[r_{half}, r]\})$,
and (3) subspace II \& III $\operatorname{span}(\{u_i v_j^T| i\in[r_{half}, m], j\in[r_{half}, n]\})$.
We compare the performance of the pre-trained model, the fine-tuned models, and modified fine-tuned models with different subspace projection strategies applied on $\Delta W$ in Figure~\ref{fig:single-task}. For more details, please refer to Appendix~\ref{sec:projection-merge}.

Figure~\ref{fig:single-task} demonstrates that projecting fine-tuned updates onto subspace I result in a slight improvement in performance on downstream tasks compared to the pre-trained model, sometimes showing no improvement at all. Projection onto subspace II leads to moderate improvement, while projection onto subspace II \& III results in significant performance gains, nearly reaching the level of the fine-tuned model.
Based on these observations, we draw the following conclusions:

\begin{tcolorbox}[width=\linewidth, colback=white!95!black]
  Fine-tuning largely maintains the most important pre-trained features, but leverages less significant dimensions for task-specific learning and activates or repurposes previously unused dimensions in the weight space.
\end{tcolorbox}

\section{Parameter Interference Between Task-Specific Models}
\label{sec:parameter_interference}

From the previous section, we build an understanding of how fine-tuning modifies the pre-trained weights to adapt to a specific downstream task.
In this section, we investigate the parameter interference between models fine-tuned on different tasks, which has been widely explored in multi-task learning and multi-task model merging, primarily within the model parameter space.
We add superscripts to denote the task index, e.g. $W_{ft}^{(i)}$ and $b_{ft}^{(i)}$ for the $i$-th task.

Assume we have $T$ tasks, and each task has a fine-tuned model. In the simplest cases, we consider the linear layers, accepting a common input $x$ and outputting $T$ different outputs $y^{(1)}, y^{(2)}, ..., y^{(T)}$.
According to Eq.(\ref{eq:ft_decomposition}) and Eq.(\ref{eq:ft_decomposition_ft_part}), each $y^{(i)}$ can be decomposed into the pre-trained part and fine-tuned part as follows:
\begin{align}
  y^{(i)} = \underbrace{\sum_{j=1}^{r} \sigma_j \langle x, v_j \rangle u_j + b}_{\text{pre-trained part}} + \underbrace{\sum_{j=1}^{m} \sum_{k=1}^{n} \delta_{jk}^{(i)} \langle x, v_k \rangle u_j + \Delta b^{(i)}}_{\text{fine-tuned part}}.
\end{align}
Where the pre-trained term is shared and remains constant during fine-tuning across all tasks.
In the context of model merging, these models are merged to construct a unified multi-task model that can perform all tasks simultaneously. A common approach is to use a weighted average of the fine-tuned weights, i.e. $W_{merged} = \sum_{l=1}^{T} \lambda_l W_{ft}^{(l)}$ and $b_{merged} = \sum_{l=1}^{T} \lambda_l b_{ft}^{(l)}$. This is equivalent to merging the fine-tuned parts of the models, while the pre-trained parts are shared across all tasks. Therefore, we express the output of the merged model as:
\begin{align}
  \label{eq:ft_decomposition_model_merging}
  y_{merged} = \text{pre-trained part} +
  \sum_{j=1}^{m} \sum_{k=1}^{n} \left(\sum_{l=1}^{T} \lambda_l \delta_{jk}^{(l)}\right) \langle x, v_k \rangle u_j + \sum_{l=1}^{T} \lambda_l \Delta b^{(l)}.
\end{align}
Substitute the input $x$ with $x^{(i)}$ from the $i$-th task (domain), we aim to minimize the discrepancy between the output of the merged model and the output of the $i$-th fine-tuned model. We formulate the optimization problem as follows:
\begin{align}
  \label{eq:model_merging_error}
  \min_{\lambda_l} \left\|y_{merged} - y^{(i)}\right\|_2^2
   & = \min_{\lambda_l} \left\|\sum_{j=1}^{m} \sum_{k=1}^{n} \left[\left(\sum_{l=1}^{T} \lambda_l \delta_{jk}^{(l)}\right) - \delta_{jk}^{(i)}\right] \langle x^{(i)}, v_k \rangle u_j + \left(\sum_{l=1}^{T} \lambda_l \Delta b^{(l)}\right) - \Delta b^{(i)} \right\|_2^2.
\end{align}
Using triangle inequality, we can decompose the error into two parts and assert an upper bound:
\begin{align}
  \left\|y_{merged} - y^{(i)}\right\|_2^2
   & \leq \underbrace{\left\|\sum_{j=1}^{m} \sum_{k=1}^{n} \left[\left(\sum_{l=1}^{T} \lambda_l \delta_{jk}^{(l)}\right) - \delta_{jk}^{(i)}\right] \langle x^{(i)}, v_k \rangle u_j \right\|_2^2}_{\text{weight term}}
  + \underbrace{\left\|\left(\sum_{l=1}^{T} \lambda_l \Delta b^{(l)}\right) - \Delta b^{(i)} \right\|_2^2}_{\text{bias term}}.
\end{align}
For the bias term, we have a closed-form solution for $\lambda_l = (\Delta B^\top \Delta B)^{-1} \Delta B \Delta b^{(i)}$, where $\Delta B$ is a matrix with $l$-th columns as $\Delta b^{(l)}$. Notice that this solution varies with the task (domain) index of the coming input $x^{(i)}$, so a more straightforward solution is to fix the bias term during the fine-tuning process, so that $\Delta b^{(i)} = 0$ for all $i$.
As for the weight term, parameter interference does not occur on subspace that is orthogonal to the input, i.e. $\operatorname{span}(\{u_i v_j^T| i\in[m], j\in[n] \text{ and } \langle x^{(i)}, v_j\rangle = 0\})$, thus a possible strategy is to enlarge the input size to increase the dimension of the orthogonal subspace. This explains why model merging methods perform better on larger models with more dimension redundancy.
When the input activates certain dimensions, i.e. for $k$ such that $\langle x^{(i)}, v_k \rangle \neq 0$, the interference is inevitable unless the domain gap between different tasks is large enough to make the activation dimensions disjoint.
Note that we can gain the same conclusion within the original model parameter space, by simply replacing the basis vectors $\{u_i\}_i$ and $\{v_i\}_i$ in this section with the standard Euclidean basis vectors $\{e_i\}_i$.

\section{Resolving Parameter Interference using Sparse Mixture of Low-Rank Experts}
\label{sec:sparse_mixture_of_experts}

\begin{figure}[t]
  \centering
  \includegraphics[width=0.85\linewidth]{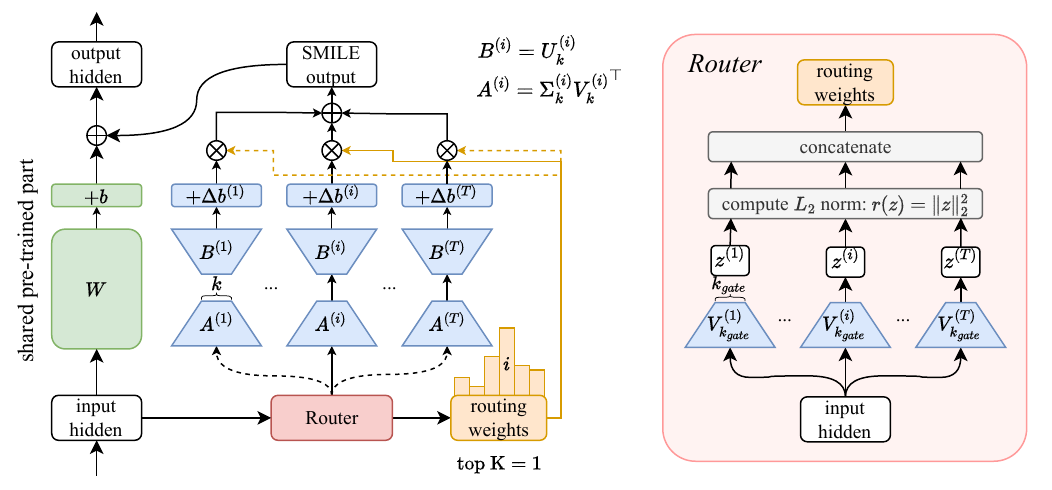}
  \caption{
    The architecture of the proposed Sparse MIxture of Low-rank Experts (SMILE) model.
  }
  \label{fig:SMILE}
\end{figure}

Understanding that addressing parameter interference by model merging is difficult, even just for the bias terms, the optimal method for weight combination has a closed-form solution that varies by task.
To manage this challenge, we introduce an innovative approach with a Sparse MIxture of Low-rank Experts (SMILE) model in this section, which operates in a zero-shot fashion, meaning no data or training is required.
An overview is shown in Figure~\ref{fig:SMILE}.

We upscale the linear layers from source models to the SMILE model, which consists of a shared pre-trained part, a router, and several low-rank experts.
Figure~\ref{fig:SMILE} is organized into two primary sections: the overall model architecture (left) and the routing mechanism (right).

Recall the output decomposition in Eq.(\ref{eq:ft_decomposition}) and Eq.(\ref{eq:ft_decomposition_model_merging}), we can express the output of a merged model as the output of the pre-trained model plus a weighted sum of the fine-tuned parts of the individual models.
If we can identify the most relevant experts for a given input, we can dynamically select the corresponding fine-tuned parts to combine with the pre-trained part. Then the merging error in Eq.(~\ref{eq:model_merging_error}) can be minimized. Mathematically, we can express this idea as:
\begin{align}
  \label{eq:naive_moe}
  y_{merged} = \text{pre-trained part} +
  \sum_{j=1}^{m} \sum_{k=1}^{n} \left(\sum_{l=1}^{T} \lambda_i\left(x^{(i)}\right) \delta_{jk}^{(l)}\right) \langle x^{(i)}, v_k \rangle u_j + \sum_{l=1}^{T} \lambda_i\left(x^{(i)}\right) \Delta b^{(l)}.
\end{align}
Here, $\lambda$ is a function that maps the input to a one-hot probability distribution over the tasks, i.e. $\lambda_j\left(x^{(i)}\right) = 1$ if $j=i$, and $\lambda_j\left(x^{(i)}\right) = 0$ otherwise.
However, a naive implementation of this idea would require a training process to learn the parameters of the router and a large number of additional parameters to store the fine-tuned weights of all tasks.
A more efficient approach is to remove less significant terms from the fine-tuned components in Eq.(\ref{eq:naive_moe}), focusing on retaining the most pertinent knowledge for each task.
Therefore, the parameter space must be ranked by the importance of its dimensions.
However, from previous findings in Section~\ref{sec:revisiting_fine_tuning}, we know that the fine-tuned information is distributed across less significant dimensions (Space II \& III), which is a large portion of the whole space.
We opt to use SVD to decompose the parameter differences $\Delta W^{(i)}$ for each task, and then apply a low-rank approximation to extract the most important part as follows:
\begin{align}
  \Delta W^{(i)} & = U^{i} \Sigma^{(i)} {V^{(i)}}^\top
  = \sum_{j=1}^{r^{(i)}} \sigma_j^{(i)} u_j^{(i)} {v_j^{(i)}}^\top
  \approx \sum_{j=1}^{k} \sigma_j^{(i)} u_j^{(i)} {v_j^{(i)}}^\top
  = U_k^{(i)} \Sigma_k^{(i)} {V_k^{(i)}}^\top, 1\leq k \leq r^{(i)}.
\end{align}
Where $r^{(i)}$ is the rank of the fine-tuned weight matrix $\Delta W^{(i)}$, and $k$ is the rank of the low-rank approximation, which is determined as a hyperparameter.
$U_k^{(i)}$ and $V_k^{(i)}$ contains the first $k$ columns of $U^{(i)}$ and $V^{(i)}$, respectively.
Here we drop the terms with indices $j > k$ in the summation, which correspond to the less significant dimensions.
Let $A^{(i)} = \Sigma_k^{(i)} {V_k^{(i)}}^\top$ and $B^{(i)} = U_k^{(i)}$, we can express the approximation similar to a LoRA adapter: $\Delta W x = B^{(i)} A^{(i)} x$.
The following theorem states the optimality of this low-rank approximation.
\begin{theorem}
  \label{thm:low_rank_approximation}
  Given a matrix $W \in \mathbb{R}^{m \times n}$, its low-rank approximation $W_k = U_k \Sigma_k V_k^\top$ with rank $k$ minimizes the Frobenius norm of the difference between $W$ and $W_k$, i.e. $W_k = \argmin_{\operatorname{rank}(W') = k} \|W - W'\|_F$.
\end{theorem}
Another key component of the SMILE model is the router, which determines the routing weights.
The routing weights should reflect the importance of each expert for a given input, and we hypothesize that the most important dimensions of the parameter updates have a larger probability of aligning with the input vector.
We provide a rationale for this hypothesis by examining the gradient flow throughout fine-tuning in Appendix~\ref{sec:gradient_flow}.
Therefore, we design the routing logits as the $L_2$ norm of the projections of the input onto the low-rank matrices. Mathematically, we can express the routing weights as:
\begin{equation}
  \label{eq:routing}
  r^{(i)} = \left\|\sum_{j=1}^{k_{gate}} \left\langle x, v_{j}^{(i)} \right\rangle \right\|_2 = \left\| V_{k_{gate}}^{(i)^\top} x \right\|_2.
\end{equation}
Where $k_{gate}$ is the number of dimensions used for routing, which is a hyperparameter.
$k_{gate}$ should not be excessively large, which could diminish the distinctiveness of routing weights.
In the extreme case where $k_{gate} = n$, $r^{(i)} = \|x\|_2$, which is equivalent to a uniform distribution over all experts.
In our hyperparameter analysis, we find that $k_{gate} = 4$ or $8$ is a good choice for most tasks.
To summarize, the output of the SMILE module can be expressed as:
\begin{align}
  y         & = \left(Wx+b\right) + \sum_{i=1}^{T} \frac{\lambda_i}{\sum_{j=1}^{T}\lambda_i} \left( U_k^{(i)} \Sigma_k^{(i)} {V_k^{(i)}}^\top x + \Delta b^{(i)}\right) \\
  \lambda_i & = \left\{ \begin{array}{ll}
                          p_i, & p_i \in \operatorname{TopK}(\{p_j\}_{j=1}^{T}, K) \\
                          0,   & \text{otherwise}
                        \end{array}
  \right.                                                                                                                                                               \\
  p_i       & = \operatorname{softmax}_i(r^{(i)})   = \operatorname{softmax}_i \left( \left\| V_{k_{gate}}^{(i)^\top} x \right\|_2 \right).
\end{align}

\textbf{Complexity Analysis}:
The linear layer has $m(n+1)$ parameters.
The upscaled SMILE module has $m(n+1) + T(mk + nk + m) + n T k_{gate}$ parameters, the additional parameters have a space complexity of $O(T(m k + n(k+ k_{gate})))$.
For every input token, an additional number of parameters of $n T k_{gate} + K (m k + n k + m)$ are activated, with $K$ representing the top-K hyperparameter.
For instance, with $T=8, k_{gate}=4, k=32, K=1$ and $m=n=1024$, the SMILE model has $565K$ additional parameters, which is about 53.9\% of the original parameter count.
$99K$ additional parameters are activated for each input token, which is only about 9.4\% of the original parameter count.

\textbf{Extending to Parameter-efficient fine-tuned (PEFT) models}:
It is straightforward to extend the SMILE upscaling to PEFT models, such as LoRA fine-tuned models.
We can still decompose the fine-tuning updates using SVD as $\Delta W_{LoRA} = B_{LoRA} A_{LoRA}$.
Note that for parameter-efficient fine-tuned models, such as LoRA fine-tuned models, $k_{gate}$ should be set to a smaller value than the hyperparameter rank of LoRA $r_{LoRA}$, and $k \leq r_{LoRA}$.

\section{Experiments}
\label{sec:experiments}

\begin{wraptable}{r}{0.52\textwidth}
  \vspace{-15pt}
  \centering
  \setlength{\tabcolsep}{2pt}
  \caption{Requirements of different model fusion methods.}
  \label{table:requirements}
  \begin{tabular}{lcccc}
    \toprule
    \textbf{Methods}      & \textbf{Validation Set} & \textbf{Test-Time Adaptation} \\
    \midrule
    Weight Averaging      & \tikzxmark              & \tikzxmark                    \\
    Fisher-Merging        & \checkmark              & \tikzxmark                    \\
    RegMean               & \checkmark              & \tikzxmark                    \\
    Task Arithmetic       & \checkmark              & \tikzxmark                    \\
    Ties-Merging          & \checkmark              & \tikzxmark                    \\
    AdaMerging            & \tikzxmark              & \checkmark                    \\
    WEMoE                 & \tikzxmark              & \checkmark                    \\
    \textbf{SMILE (Ours)} & \tikzxmark              & \tikzxmark                    \\
    \bottomrule
  \end{tabular}
  \vspace{-10pt}
\end{wraptable}

In this section, we evaluate the effectiveness of the proposed SMILE on a variety of setups, including image classification and text generation tasks, as well as full fine-tuning and LoRA fine-tuning.
Detailed information about the fine-tuned models is in Appendix~\ref{sec:individuals}.
We compare our method with several SOTA model fusion techniques, including Simple Averaging~\citep{wortsmanModelSoupsAveraging2022}, Fisher merging~\citep{michaelmatenaMergingModelsFisherWeighted2022}, RegMean~\citep{xisenjinDatalessKnowledgeFusion2023}, Task Arithmetic~\citep{ilharcoEditingModelsTask2023}, Ties-Merging~\citep{yadavResolvingInterferenceWhen2023}, AdaMerging~\citep{yangAdaMergingAdaptiveModel2023}, and WEMoE~\citep{tangMergingMultiTaskModels2024}.
To further demonstrate the scalability of SMILE upscaling, we also conduct experiments using off-the-shelf large language models fine-tuned from \texttt{Mistral-7B-v0.1}.
Our code implementation is based on FusionBench~\citep{tangFusionBenchComprehensiveBenchmark2024}.

\subsection{Multi-Task Model Fusion on Open-Vocabulary Image Classification Tasks}
\label{sec:open-vocabulary-image-classification}

\begin{wrapfigure}{r}{0.4\textwidth}
  \centering
  \includegraphics[width=0.4\textwidth]{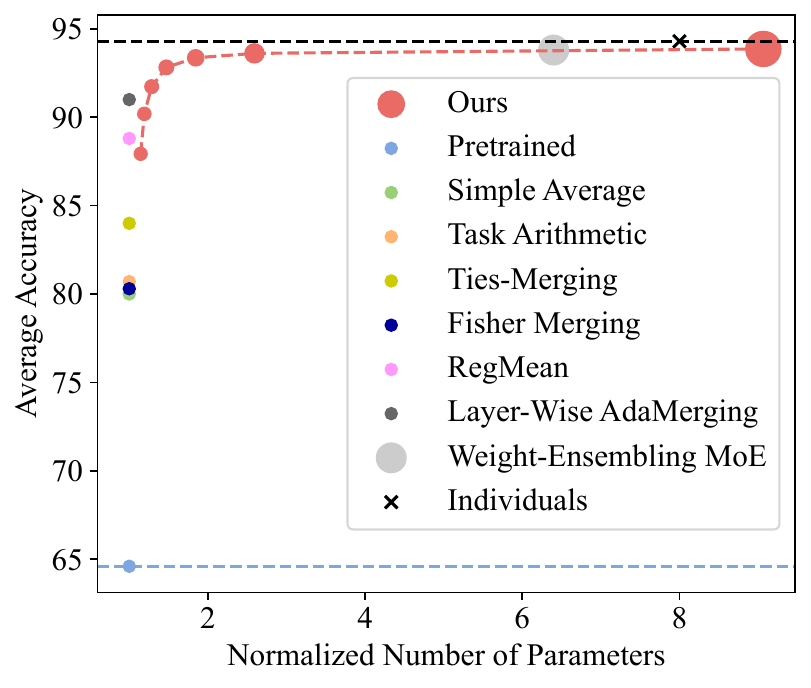}
  \caption{
    Multi-task model fusion experiment on eight image classification tasks using CLIP-ViT-L/14 models ($k_{gate}=16$).
  }
  \vspace{-20pt}
  \label{fig:clip-vit-l-14_scatter}
\end{wrapfigure}
We first evaluate our proposed SMILE method on eight diverse open-vocabulary image classification tasks using CLIP models from HuggingFace~\footnote{\url{https://huggingface.co/openai/clip-vit-base-patch32}}\footnote{\url{https://huggingface.co/openai/clip-vit-large-patch14}}.
For each task, the text encoder of the pre-trained model is frozen, and only the vision encoder is fine-tuned.
Table~\ref{table:requirements} presents the requirements of different model fusion methods, highlighting that SMILE is a training-free model fusion method that does not require additional labeled samples or test-time adaptation.

Figures~\ref{fig:clip-vit-b-32_scatter} and \ref{fig:clip-vit-l-14_scatter} illustrates the average accuracy of the merged model across different methods, for SMILE $k_{gate}$ is set to 16 and $k$ is varied from 4 to 128. These two figures demonstrate the effectiveness of SMILE and the trade-off between performance and model size.

In Tables~\ref{table:clip-vit-b-32} and \ref{table:clip-vit-l-14}, we compare the performance of various model fusion methods on CLIP-ViT-B/32 and CLIP-ViT-L/14 models, respectively.
Our SMILE method achieves competitive results across all tasks.
For instance, with CLIP-ViT-L/14 models, SMILE (2: $k_{gate}=16$, $k=128$) achieves 99.3\% of the individual model performance while using only 2.56 times the parameters of a single model, compared to eight individual fine-tuned models and Weight-Ensembling MoE which requires 6.40 times the parameters.

\begin{table}
  \caption{
    Multi-task model fusion methods on eight image classification tasks using CLIP-ViT-B/32 models.
    Here we show two different hyperparameter settings for our method:
    (1) $k_{gate}=16$, $k=32$ and the normalized parameter count is 1.61;
    (2) $k_{gate}=16$, $k=128$ and the normalized parameter count is 3.07.
  }
  \label{table:clip-vit-b-32}
  \centering
  \setlength{\tabcolsep}{2pt}
  \begin{tabular}{lccccccccc}
    \toprule
    \textbf{Method}                  & \textbf{SUN397} & \textbf{Cars} & \textbf{RESISC45} & \textbf{EuroSAT} & \textbf{SVHN} & \textbf{GTSRB} & \textbf{MNIST} & \textbf{DTD}  & \textbf{Avg.}          \\
    \midrule
    Individual                       & 75.0            & 78.3          & 95.2              & 99.0             & 97.3          & 98.9           & 99.6           & 79.7          & 90.3 (100\%)           \\
    \midrule
    Simple Averaging                 & 65.4            & 62.6          & 70.8              & 76.9             & 64.5          & 54.9           & 86.3           & 50.9          & 66.5 (73.6\%)          \\
    Fisher Merging                   & 66.7            & 64.0          & 72.2              & 91.6             & 69.0          & 64.3           & 83.5           & 53.7          & 70.6 (78.2\%)          \\
    RegMean                          & 67.8            & 68.9          & 82.5              & 94.4             & 90.6          & 79.2           & 97.6           & 63.2          & 80.5 (89.1\%)          \\
    Task Arithmetic                  & 57.1            & 55.7          & 64.9              & 76.7             & 77.9          & 68.5           & 96.1           & 47.2          & 68.0 (75.3\%)          \\
    Ties-Merging                     & 67.1            & 64.2          & 74.1              & 76.8             & 77.7          & 69.4           & 94.1           & 54.0          & 72.2 (80.0\%)          \\
    AdaMerging                       & 67.9            & 71.3          & 83.5              & 92.7             & 87.4          & 92.9           & 98.2           & 67.0          & 82.6 (91.5\%)          \\
    WEMoE ($\times$ 6.27)            & \textbf{73.7}   & 76.8          & \textbf{93.4}     & 98.2             & 96.8          & \textbf{98.2}  & \textbf{99.6}  & 76.6          & 89.2 (98.8\%)          \\
    \textbf{SMILE} (1, $\times$1.61) & 73.6            & 74.4          & 89.5              & 98.1             & 95.4          & 97.3           & 99.5           & 76.3          & 87.7 (97.1\%)          \\
    \textbf{SMILE} (2, $\times$3.07) & 73.6            & \textbf{77.8} & 92.0              & \textbf{98.3}    & \textbf{96.9} & 98.1           & \textbf{99.6}  & \textbf{78.1} & \textbf{89.3 (98.9\%)} \\
    \bottomrule
  \end{tabular}
\end{table}

\begin{table}
  \caption{
    Multi-task model fusion methods on eight image classification tasks using CLIP-ViT-L/14 models.
    Here we show two different hyperparameter settings for our method:
    (1) $k_{gate}=16$, $k=32$ and the normalized parameter count is 1.47;
    (2) $k_{gate}=16$, $k=128$ and the normalized parameter count is 2.56.
  }
  \label{table:clip-vit-l-14}
  \centering
  \setlength{\tabcolsep}{2pt}
  \begin{tabular}{lccccccccc}
    \toprule
    \textbf{Method}                  & \textbf{SUN397} & \textbf{Cars} & \textbf{RESISC45} & \textbf{EuroSAT} & \textbf{SVHN} & \textbf{GTSRB} & \textbf{MNIST} & \textbf{DTD}  & \textbf{Avg.}          \\
    \midrule
    Individual                       & 82.8            & 92.9          & 97.4              & 99.2             & 97.9          & 99.2           & 99.8           & 85.5          & 94.3 (100\%)           \\
    \midrule
    Simple Averaging                 & 72.5            & 81.5          & 82.2              & 90.0             & 81.6          & 74.0           & 96.6           & 61.8          & 80.0 (84.8\%)          \\
    Fisher Merging                   & 70.6            & 79.4          & 84.1              & 98.1             & 74.7          & 85.0           & 89.5           & 61.0          & 80.3 (85.2\%)          \\
    RegMean                          & 75.3            & 88.4          & 90.0              & 97.1             & 95.9          & 92.4           & 98.5           & 72.6          & 88.8 (94.2\%)          \\
    Task Arithmetic                  & 72.0            & 79.0          & 80.5              & 86.0             & 87.5          & 83.5           & 98.0           & 58.8          & 80.7 (85.6\%)          \\
    Ties-Merging                     & 74.7            & 83.3          & 86.4              & 91.3             & 89.7          & 85.2           & 97.8           & 63.9          & 84.0 (89.1\%)          \\
    AdaMerging                       & 78.1            & 90.7          & 90.8              & 96.5             & 94.8          & 97.5           & 98.6           & 81.3          & 91.0 (96.5\%)          \\
    WEMoE   ($\times$6.40)           & 81.5            & \textbf{92.3} & \textbf{96.5}     & 98.8             & 97.6          & \textbf{99.4}  & 99.6           & \textbf{84.5} & \textbf{93.8 (99.5\%)} \\
    \textbf{SMILE} (1, $\times$1.47) & 79.9            & 91.0          & 94.3              & 99.0             & 97.9          & 98.6           & \textbf{99.7}  & 82.2          & 92.8 (98.4\%)          \\
    \textbf{SMILE} (2, $\times$2.56) & \textbf{81.9}   & \textbf{92.3} & 95.5              & \textbf{99.1}    & \textbf{98.0} & 98.9           & \textbf{99.7}  & 83.6          & 93.6 (99.3\%)          \\
    \bottomrule
  \end{tabular}
\end{table}

\subsection{Multi-Task Model Fusion on Text Generation Tasks}
\label{sec:multi-task-model-fusion}

We further evaluate SMILE on text generation tasks using Flan-T5-base models~\footnote{\url{https://huggingface.co/google/flan-t5-base}}, which are fine-tuned on eight tasks from the GLUE benchmark~\citep{wangGLUEMultiTaskBenchmark2018}.
We use two different fine-tuning strategies: full fine-tuning and LoRA fine-tuning with $r_{LoRA}=16$.
We present the results in Tables~\ref{table:flan-t5-base_fusion_full_finetuned} and \ref{table:flan-t5-base_fusion_lora16}, for full fine-tuned models and LoRA fine-tuned models, respectively.
For fully fine-tuned models, SMILE consistently outperforms other fusion methods across all eight tasks. With just 1.52 times the parameters of a single model, SMILE (2: $k_{gate}=8,k=32$) achieves 99.0\% of the individual model performance with 1.52 times the parameters of a single model. In the LoRA fine-tuned scenario, SMILE maintains strong performance with minimal parameter increase (1.02 times). It achieves 99.3\% of the individual model performance, significantly surpassing other multi-task model fusion methods.

\begin{table}
  \caption{
    Multi-task performance when merging Flan-T5-base (full fine-tuned) models on all eight tasks.
    Here we show two different hyperparameter settings for our method:
    (1) $k_{gate}=4,k=16$ and the normalized parameter count is 1.26;
    (2) $k_{gate}=8,k=32$ and the normalized parameter count is 1.52.
  }
  \label{table:flan-t5-base_fusion_full_finetuned}
  \centering
  \begin{tabular}{lccccccccc}
    \toprule
    \textbf{Method}                  & \textbf{CoLA} & \textbf{MNLI} & \textbf{MRPC} & \textbf{QNLI} & \textbf{QQP}  & \textbf{RTE}  & \textbf{SST2} & \textbf{STSB} & \textbf{Avg.}          \\
    \midrule
    Individual                       & 75.0          & 83.4          & 87.5          & 91.5          & 85.4          & 85.9          & 93.6          & 88.7          & 86.4 (100\%)           \\
    \midrule
    Weight Averaging                 & 69.1          & 62.6          & 79.4          & 89.8          & 83.9          & 81.2          & 91.7          & 73.2          & 78.9 (91.3\%)          \\
    Task Arithmetic                  & 70.5          & 57.8          & 78.4          & 90.2          & 83.6          & 80.5          & 92.3          & 77.8          & 78.9 (91.3\%)          \\
    Ties-Merging                     & 70.3          & 65.0          & 78.9          & 90.2          & 83.5          & 81.6          & 91.7          & 78.3          & 79.9 (92.5\%)          \\
    \textbf{SMILE} (1, $\times$1.26) & 72.0          & \textbf{84.2} & 84.3          & 91.3          & 84.7          & 84.1          & 93.3          & 87.0          & 85.1 (98.5\%)          \\
    \textbf{SMILE} (2, $\times$1.52) & \textbf{73.2} & \textbf{84.2} & \textbf{85.0} & \textbf{91.3} & \textbf{84.9} & \textbf{84.8} & \textbf{93.5} & \textbf{87.3} & \textbf{85.5 (99.0\%)} \\
    \bottomrule
  \end{tabular}
\end{table}

\begin{table}
  \caption{
    Multi-task performance when merging Flan-T5-base (LoRA fine-tuned) models on all eight tasks.
    We choose $k_{gate}=2, k=4$ and the normalized parameter count is 1.02.
  }
  \label{table:flan-t5-base_fusion_lora16}
  \centering
  \begin{tabular}{lccccccccc}
    \toprule
    \textbf{Method}               & \textbf{CoLA} & \textbf{MNLI} & \textbf{MRPC} & \textbf{QNLI} & \textbf{QQP}  & \textbf{RTE}  & \textbf{SST2} & \textbf{STSB} & \textbf{Avg.}          \\
    \midrule
    Individual                    & 69.1          & 82.7          & 85.5          & 90.9          & 84.0          & 84.4          & 92.9          & 87.4          & 84.6 (100\%)           \\
    \midrule
    Weight Averaging              & \textbf{69.7} & 59.7          & 78.9          & 90.1          & 83.8          & 80.5          & 91.2          & 72.0          & 78.2 (92.4\%)          \\
    Task Arithmetic               & 68.8          & 55.2          & 78.7          & 89.8          & 83.7          & 79.1          & 91.5          & 72.4          & 77.4 (91.5\%)          \\
    Ties-Merging                  & 68.3          & 56.3          & 79.4          & 89.8          & 83.7          & 79.4          & 91.6          & 71.2          & 77.5 (91.6\%)          \\
    \textbf{SMILE} ($\times$1.02) & 69.3          & \textbf{82.9} & \textbf{83.8} & \textbf{90.6} & \textbf{83.9} & \textbf{83.4} & \textbf{93.1} & \textbf{85.1} & \textbf{84.0 (99.3\%)} \\
    \bottomrule
  \end{tabular}
\end{table}

\subsection{Spare Mixture of Low-Rank Experts Analysis}
\label{sec:method_analysis}

To better understand SMILE, we further conduct ablation studies using CLIP and Flan-T5 models.

\begin{figure}[t]
  \begin{center}
    \begin{subfigure}{115pt}
      \centering
      \includegraphics[height=1.2in]{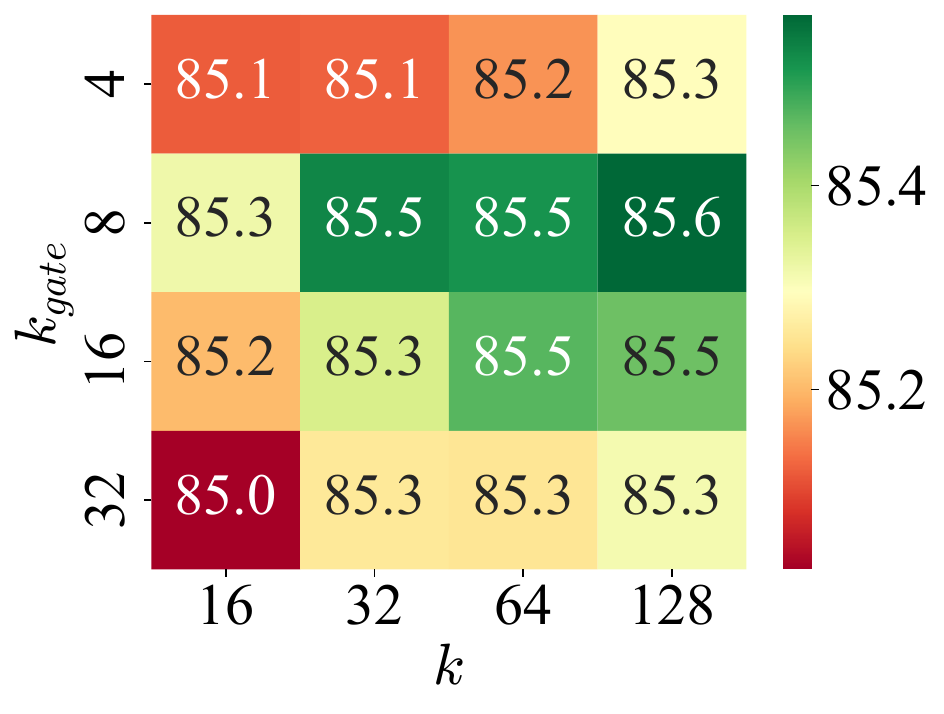}
      \caption{average performance}
      \label{fig:flan-t5-base_hp-acc}
    \end{subfigure}%
    \begin{subfigure}{115pt}
      \centering
      \includegraphics[height=1.2in]{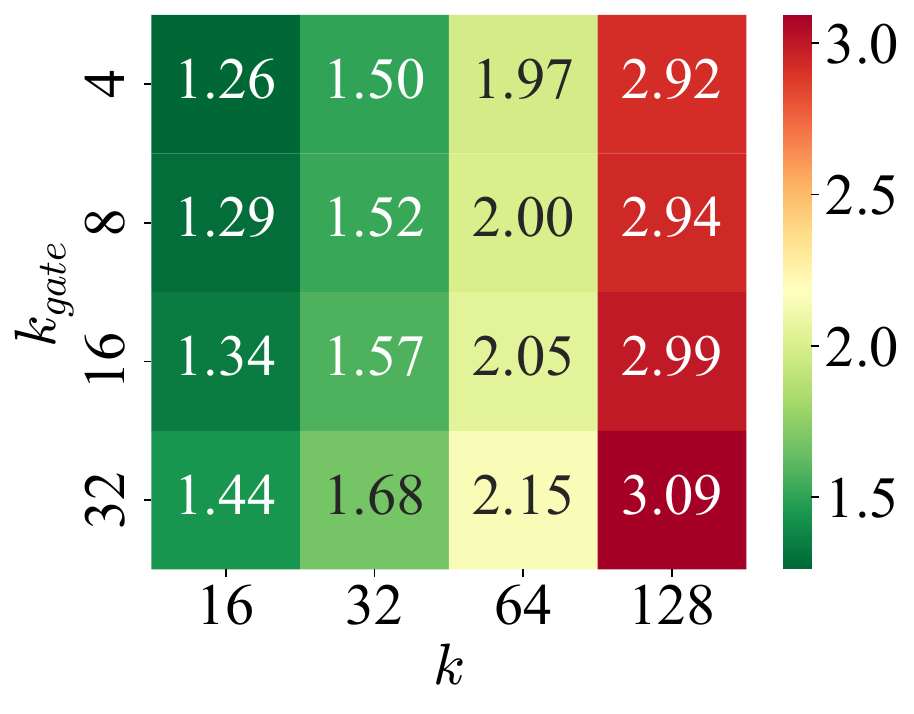}
      \caption{normalized parameters count}
      \label{fig:flan-t5-base_hp-params}
    \end{subfigure}
    \begin{subfigure}{115pt}
      \centering
      \includegraphics[height=1.2in]{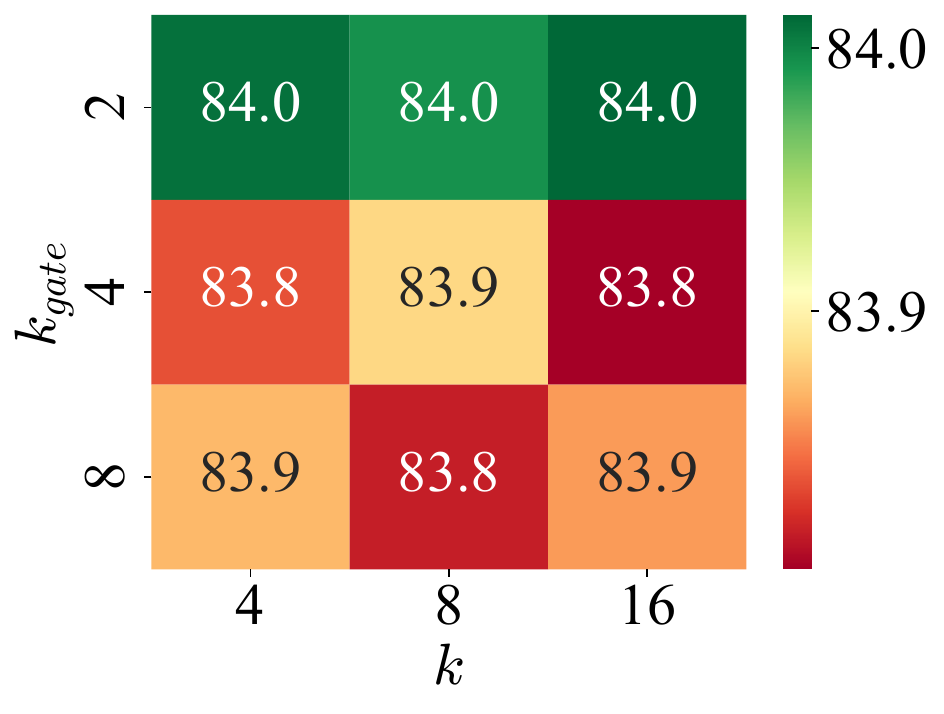}
      \caption{average performance}
    \end{subfigure}%
    \begin{subfigure}{115pt}
      \centering
      \includegraphics[height=1.2in]{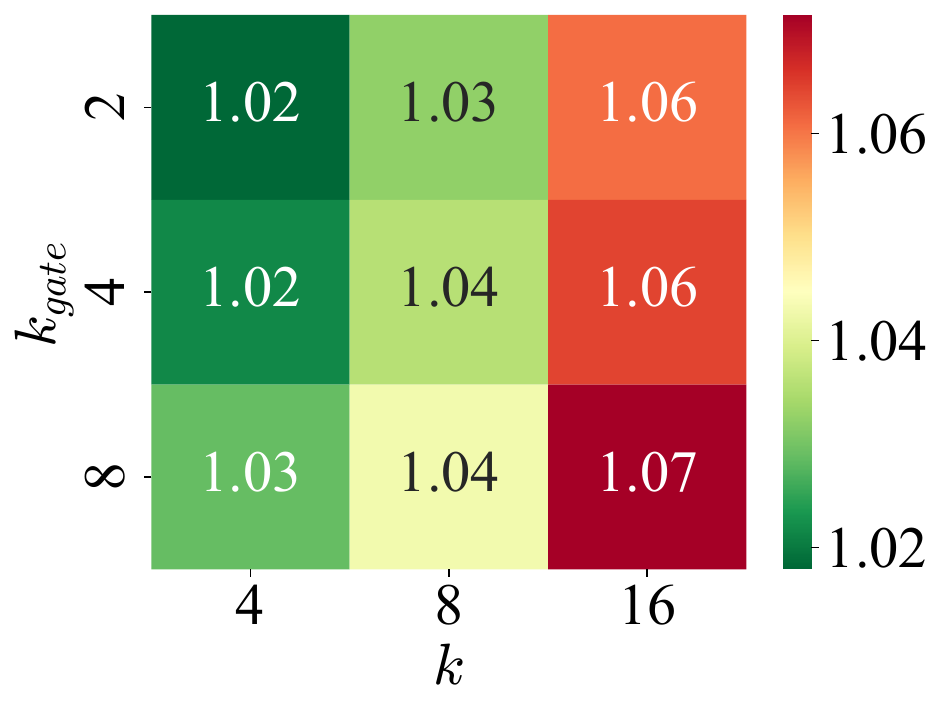}
      \caption{normalized parameters count}
      \label{fig:flan-t5-base_lora16_hp-params}
    \end{subfigure}
  \end{center}
  \caption{
    Hyperparameter analysis of the Flan-T5-Base models on eight tasks from GLUE benchmark.
    We show how different values of hyperparameters $k$ and $k_{gate}$ affect the average performance and the normalized number of parameters in the upscaled model.
    Subfigures (a), and (b) show the results of the full fine-tuned models, while subfigures (c), and (d) show the results of the fine-tuned models with $r_{LoRA}=16$.
  }
  \label{fig:flan-t5-base_hp}
\end{figure}
\begin{figure}
  \centering
  \includegraphics[width=0.8\textwidth]{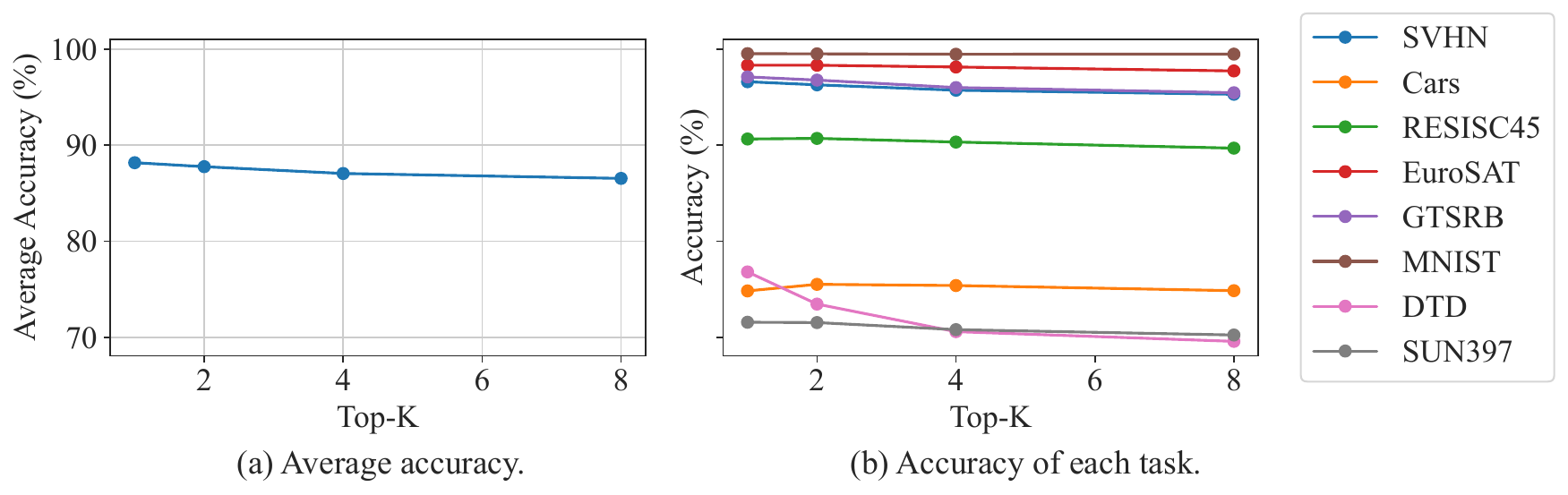}
  \caption{
    Ablations on the Top-$K$ routing for CLIP-ViT-B/32 models on eight image classification tasks ($k_{gate}=16, k=32$).
    Here we show the average accuracy and the accuracy on each task, and the y-axis is shared.
  }
  \label{fig:clip-vit-b-32_ablations-topk}
\end{figure}

\textbf{Ablations on the low-rank approximation rank $k$ and routing dimension $k_{gate}$ (hyperparameter analysis).}
Our hyperparameter analysis demonstrates the flexibility and robustness of SMILE across different model architectures and tasks.
Figure \ref{fig:flan-t5-base_hp} illustrates the impact of hyperparameters $k$ and $k_{gate}$ on performance and parameter count for Flan-T5-Base models in both full and LoRA fine-tuned scenarios. For CLIP-ViT models, Figures \ref{fig:clip-vit-b-32_hp} and \ref{fig:clip-vit-l-14_hp} provide detailed heatmaps and line plots showing the relationship between hyperparameters, average accuracy, and parameter count.
Across all models, we observe a consistent trend: increasing $k$ and $k_{gate}$ generally leads to improved performance, but with diminishing returns as parameter count grows. Notably, SMILE achieves near-optimal performance with relatively small values of $k$ and $k_{gate}$. This analysis highlights the effectiveness of SMILE in balancing performance and efficiency, allowing users to fine-tune the trade-off based on their specific requirements. The stability of performance across a range of hyperparameter values also underscores the robustness of our method, making it adaptable to various multi-task fusion scenarios. For more details, please refer to Appendix~\ref{sec:hyperparameter-analysis}.

\textbf{Ablations on Top-$K$ routing (routing analysis).}
Here we compare different values $K$ in the top-$K$ routing mechanism.
The plots in Figure~\ref{fig:clip-vit-b-32_ablations-topk} illustrate the impact of varying $K$ on both the average accuracy across all tasks (Figure~\ref{fig:clip-vit-b-32_ablations-topk}a) and the accuracy of each individual task (Figure~\ref{fig:clip-vit-b-32_ablations-topk}b) when using the CLIP-ViT-B/32 model across eight image classification tasks.
We observe that the average accuracy decreases slightly as $K$ increases.
This suggests that larger values of $K$, which allow more experts to be used for each input, are not necessary for multi-task model fusion where each expert is specialized for a specific task.
In general, the performance of individual tasks is relatively stable across different values of $K$, indicating the robustness of the routing mechanism of SMILE (Equation~(\ref{eq:routing})).

\subsection{Scalability to Large-Scale Models (Mistral-7B models)}
\label{sec:scalability}

\begin{wrapfigure}{r}{0.4\textwidth}
  \centering
  \vspace{-15pt}
  \includegraphics[width=0.4\textwidth]{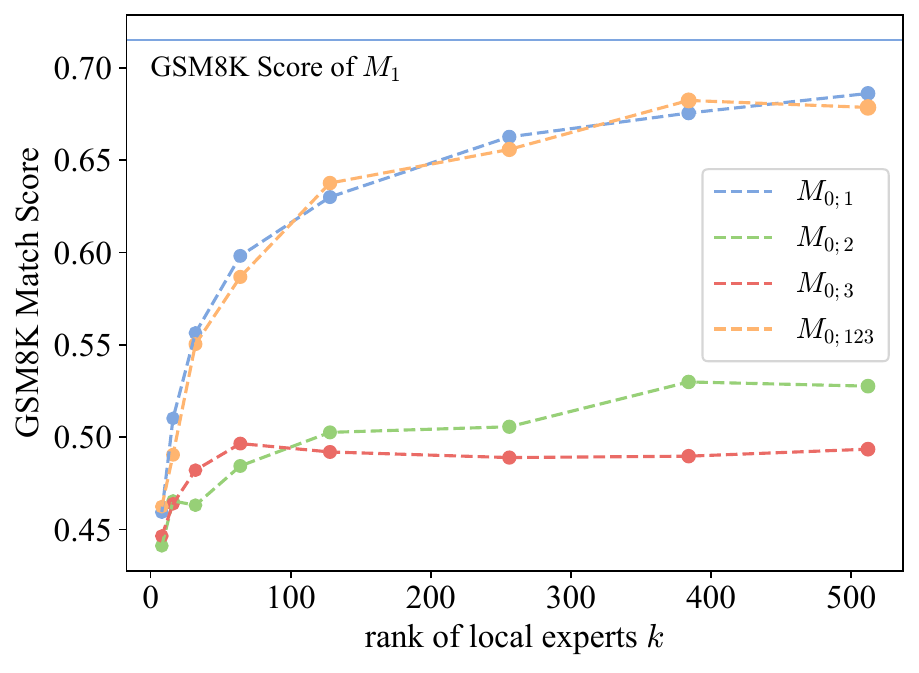}
  \caption{
    The GSM8K benchmark scores of upscaled Mistral-7B models with varying $k$.
  }
  \label{fig:mistral-7b_gsm8k}
  \vspace{-20pt}
\end{wrapfigure}
To demonstrate the scalability of SMILE to large-scale models, we conduct experiments on \texttt{Mistral-7B} models.
We use \texttt{Mistral-7B-v0.1} as our base pre-trained model, referred to as $M_{0}$, and acquire three specialized models from HuggingFace~\citep{wolfHuggingFaceTransformersStateoftheart2020}.
The expert models are respectively labeled as $M_1$, $M_2$, and $M_3$~\footnote{The expert models are
  \href{https://huggingface.co/meta-math/MetaMath-Mistral-7B}{\texttt{meta-math/MetaMath-Mistral-7B}},
  \href{https://huggingface.co/cognitivecomputations/dolphin-2.1-mistral-7b}{\texttt{cognitivecomputations/dolphin-2.1-mistral-7b}} and
  \href{https://huggingface.co/uukuguy/speechless-code-mistral-7b-v1.0}{\texttt{uukuguy/speechless-code-mistral-7b-v1.0}} respectively.}.
Of these models, $M_1$ stands out as the sole expert in mathematics.
To demonstrate the efficacy of the zero-shot routing mechanisms, we construct four distinct series of SMILE models with various expert combinations. These series are designated as $M_{0;1}$, $M_{0;2}$, $M_{0;3}$, and $M_{0;123}$, with $M_{0;i_1 \dots i_n}$ indicating the SMILE model that combines $M_0$ with the expert models $M_{i_1}, \dots, M_{i_n}$.

\begin{table}
  \centering
  \caption{
    Comparison of individual Mistral-7B models and the upscaled model on various benchmark tasks.
  }
  \label{table:mistral-7b}
  \begin{tabular}{ccccc}
    \toprule
    \textbf{Model}                                  & \textbf{MMLU}  & \textbf{TruthfulQA} & \textbf{GSM8K} & \textbf{ARC Challenge} \\
    \midrule
    $M_0$ (pre-trained)                             & 59.64          & 42.62               & 38.81          & 53.92                  \\
    $M_1$                                           & 60.56          & 44.79               & \textbf{71.49} & 51.02                  \\
    $M_2$                                           & 60.56          & \textbf{55.88}      & 56.93          & 57.00                  \\
    $M_3$                                           & \textbf{61.18} & 47.47               & 48.98          & \textbf{57.68}         \\
    \midrule
    $M_{0;123}$ $(11.2\text{B}, k_{gate}=8, k=512)$ & 60.66          & 52.79               & 67.85          & 54.35                  \\
    \midrule
    \texttt{Qwen1.5-14B} (reference)                & 66.11          & 52.00               & 69.37          & 49.93                  \\
    \bottomrule
  \end{tabular}
\end{table}

In Figure~\ref{fig:mistral-7b_gsm8k}, we present the GSM8K benchmark scores of the upscaled Mistral-7B models with varying rank of local experts $k$ with a constant rank of routers $k_{gate}=8$. This plot highlights the trade-offs in selecting expert rank $k$ for the upscaled SMILE model, where the GSM8K score generally improves as $k$ increases, but this improvement is more pronounced for specific expert combinations, particularly for $M_{0;1}$ and $M_{0;123}$. This suggests that the routers succeed in selecting the proper expert for math problems.
In Table~\ref{table:mistral-7b}, we compare the performance of individual models and the upscaled model on various benchmark tasks. Notably, the individual expert models show strengths in specific benchmarks, such as  $M_1$ excelling in the GSM8K benchmark and $M_3$ in the ARC Challenge. This indicates that each expert brings specialized knowledge, which when combined, enhances the overall performance in a diverse set of tasks.
More details are in Appendix~\ref{appendix:large-scale-experiments}.

\section{Related Work}

\textbf{Mixture of Experts.}
The concept of Mixture of Experts (MoE) is first introduced by~\cite{jacobsAdaptiveMixturesLocal1991}, involving training multiple specialized models.
This concept has gained significant attention in recent years~\citep{jiangMixtralExperts2024,daiDeepSeekMoEUltimateExpert}, with much of the innovation focusing on routing mechanisms and expert design.
Much innovation revolves around the design of more efficient routers.
For example, the Switch Transformer \citep{fedusSwitchTransformersScaling2022} selects only the top expert for each token, simplifying the process and improving scalability.
Similarly, \citep{lewisBASELayersSimplifying2021} use a linear assignment to optimize token-expert affinities, ensuring an equal spread of tokens among experts.
In~\citep{ostapenkoModularLLMsBuilding2024}, the authors propose to build a library of LoRA adapters using model-based clustering and build a MoE to select the most relevant adapters based on input without retraining.
For detailed reviews on MoE, see \citep{fedusReviewSparseExpert2022}, and for MoE in the context of model merging (MoEerging), refer to \citep{yadavSurveyModelMoErging2024}.

\textbf{Deep Model Fusion.}
Mode connectivity reveals that different model solutions can be linked by low-loss path in the parameter space~\citep{danielfreemanTopologyGeometryHalfrectified2017,nagarajanUniformConvergenceMay2019,draxlerEssentiallyNoBarriers2019,frankleLinearModeConnectivity2020,entezariRolePermutationInvariance2022,garipovLossSurfacesMode2018,tatroOptimizingModeConnectivity2020,yunisConvexityLinearMode2022,bentonLossSurfaceSimplexes2021}, facilitating model fusion by weight interpolation~\citep{izmailovAveragingWeightsLeads2019,michaelmatenaMergingModelsFisherWeighted2022,wortsmanModelSoupsAveraging2022,kaddourStopWastingMy2022,ilharcoEditingModelsTask2023,yadavResolvingInterferenceWhen2023,yangAdaMergingAdaptiveModel2023,wuPiTuningTransferring2023}.
However, this strategy also poses challenges, particularly when merging models with diverse structures.
Alignment helps reduce model disparities by matching and interpolating components~\citep{liConvergentLearningDifferent2016,tatroOptimizingModeConnectivity2020}.
Methods involve matching activations or weights \citep{georgestoicaZipItMergingModels2023,xisenjinDatalessKnowledgeFusion2023,yangRepresentationSurgeryMultiTask2024}, using channel-wise graph matching \citep{liuDeepNeuralNetwork2022a}, or applying permutation invariance \citep{ainsworthGitReBasinMerging2023}.
Another line of research is model mixing, which combines models through gating mechanisms or depth concatenation~\citep{tangMergingMultiTaskModels2024,luTwinMergingDynamicIntegration2024,tangEfficientParetoSet2024,kimSOLAR107B2023}, allowing for more flexible and adaptive fusion strategies.

\section{Conclusion, Limitations, and Future Work}
\label{sec:conclusion}

In this paper, we introduced the Sparse Mixture of Low-Rank Experts (SMILE) model as a novel approach to model fusion, Our method leverages a zero-shot mechanism, eliminating the need for additional training data or processes, which makes it highly practical.
While the MoE method is designed to be efficient through sparse activation, it still adds extra computational overhead, especially as the number of tasks or experts increases.

Understanding which subspaces contribute most to task-specific performance could lead to more targeted and efficient fine-tuning strategies, potentially focusing on updating specific parts of the model while leaving others intact.
Additionally, this approach might be applied to other areas, like multi-modal large language models, where different types of data (modalities) are treated as separate experts.
Furthermore, it would be worth exploring how SMILE can manage multi-objective optimization by adjusting the importance of different routing weights.
Moreover, develop methods to dynamically adjust the number of experts $K$ per token based on the input, potentially improving efficiency without sacrificing performance.

\clearpage
\bibliographystyle{plainnat}
\bibliography{references}

\clearpage
\appendix
\section{Projecting Fine-tuned Updates onto Subspaces}
\label{sec:projection-merge}

This section provides an in-depth mathematical explanation of the projection merge experiments discussed in Section~\ref{sec:revisiting_fine_tuning}.
These experiments aim to gain empirical insights into the distribution of task-specific information across different subspaces of the weight matrix after fine-tuning a pre-trained model on downstream tasks.

Let $W \in \mathbb{R}^{m \times n}$ be the weight matrix of a linear layer in the pre-trained model, and $W_{ft} \in \mathbb{R}^{m \times n}$ be the corresponding weight matrix after fine-tuning. We define the weight update as $\Delta W = W_{ft} - W$.
We begin by performing a full Singular Value Decomposition (SVD) on the pre-trained weight matrix $W$:
\begin{equation}
  W = U\Sigma V^\top
\end{equation}
where $U \in \mathbb{R}^{m \times m}$ and $V \in \mathbb{R}^{n \times n}$ are orthonormal matrices containing left and right singular vectors, respectively, and $\Sigma \in \mathbb{R}^{m \times n}$ contains the singular values in descending order.
The first $r$ diagonal elements of $\Sigma$ are non-zero, where $r = \operatorname{rank}(W)$, while the remaining elements are zero.
According to Theorem~\ref{thm:matrix_subspace}, we can leverage the properties of singular value decomposition (SVD) to gain a deeper understanding of the fine-tuning process. This theorem states that any matrix $A \in \mathbb{R}^{m \times n}$ can be decomposed into a sum of rank-one matrices using the left singular vectors $\{u_i\}_{i=1}^{m}$ and right singular vectors $\{v_i\}_{i=1}^{n}$ as bases:
\begin{equation}
  A = \sum_{i=1}^{m}\sum_{j=1}^{n} \alpha_{ij} u_i v_j^\top = U \Sigma_A V^\top
\end{equation}
where $\alpha_{ij} = \langle A, u_i v_j^\top \rangle$ is the projection of $A$ onto the basis $u_i v_j^\top$, $\Sigma_A$ is a real matrix and $\Sigma_A(i, j) = \alpha_{ij}$.
This decomposition provides a powerful framework for analyzing the fine-tuning process.
When we fine-tune a pre-trained model, we can interpret the weight updates $\Delta W$ as modifications to the singular value matrix $\Sigma$, while the singular vectors $U$ and $V$ remain constant.
Then we partition the singular matrix $\Sigma$ into three zones:
\begin{itemize}
  \item \textbf{Zone I \& Subspace I:}
        $\{1, \ldots, r_{half}\}$, where $r_{half}$ is chosen such that $\sum_{i=1}^{r_{half}} \sigma_i \approx \frac{1}{2}\sum_{i=1}^{r} \sigma_i$.
        The basis of this subspace is $\{ u_i v_j^\top | 1\leq i,j \leq r_{half} \}$. The projection merged weights in this subspace can be computed as follows:
        \begin{equation}
          W_I = W + \sum_{i=1}^{r_{half}} \sum_{j=1}^{r_{half}} \langle \Delta W, u_i v_j^\top \rangle u_i v_j^\top
          = W + U_{r_{half}} U_{r_{half}}^\top \Delta W V_{r_{half}} V_{r_{half}}^\top.
        \end{equation}
        Where $U_{r_{half}}$ and $V_{r_{half}}$ are the first $r_{half}$ columns of $U$ and $V$, respectively.
  \item \textbf{Zone II \& Subspace II:}
        $\{r_{half}+1, \ldots, r\}$, where $r = \operatorname{rank}(W)$.
        The basis of this subspace is $\{u_i v_i^\top\}_{i=r_{half}+1}^{r}$. The basis of subspace II \& III is $\{u_i v_j^\top | r_{half}+1\leq i \leq r, r_{half}+1\leq j \leq r\}$. The projection merged weights in this subspace can be computed as follows:
        \begin{equation}
          W_{II} = W + \sum_{i=r_{half}+1}^{r} \sum_{j=r_{half}+1}^{r} \langle \Delta W, u_i v_j^\top \rangle u_i v_j^\top
          = W + U_{r_{half}+1:r} U_{r_{half}+1:r}^\top \Delta W V_{r_{half}+1:r} V_{r_{half}+1:r}^\top.
        \end{equation}
        Where $U_{r_{half}+1:r}$ and $V_{r_{half}+1:r}$ are the $(r_{half}+1)$-th to $r$-th columns of $U$ and $V$, respectively.
  \item \textbf{Zone III \& Subspace II + III:}
        The basis of this subspace is $\{u_i v_j^\top | r+1\leq i \leq m, r+1\leq j \leq n\}$ and the projection merged weights in this subspace can be computed as follows:
        \begin{equation}
          W_{II+III} = W + \sum_{i=r+1}^{m} \sum_{j=r+1}^{n} \langle \Delta W, u_i v_j^\top \rangle u_i v_j^\top
          = W + U_{r+1:m} U_{r+1:m}^\top \Delta W V_{r+1:n} V_{r+1:n}^\top.
        \end{equation}
        Where $U_{r+1:m}$ and $V_{r+1:n}$ are the $(r+1)$-th to $m$-th columns of $U$ and $(r+1)$-th to $n$-th columns of $V$.
\end{itemize}

We then evaluate the performance of these modified weight matrices on the downstream tasks.
The accuracy comparison in Figure~\ref{fig:single-task} is obtained by using these modified weight matrices in place of the original pre-trained weights.
For other layers instead of the linear layers, we keep the pre-trained weights unchanged.

\section{The Gradient Flow During Fine-tuning}
\label{sec:gradient_flow}

In this section, we analyze the gradient flow during the fine-tuning process to gain insights into how linear layers in a deep neural network adapt to new tasks.
We decompose a fine-tuned deep neural network $f$ into three components: the pre-linear layers $f_{pre}$, the linear layers $f_{linear}(x) = W_{ft} x + b_{ft}$, and the post-linear layers $f_{post}$. Therefore, the output of the network can be expressed as $f(x) = f_{post}(f_{linear}(f_{pre}(x)))$.
Without loss of generality, the pre-linear layers $f_{pre}$ can be dropped, as our focus is on the fine-tuning process of the linear layer $f_{linear}$.

During fine-tuning, the model parameters are updated by minimizing a loss function $\mathcal{L}$ with respect to the model weights.
Using stochastic gradient descent (SGD) as the optimization algorithm, the weight update $\Delta W$ and bias update $\Delta b$ at each step can be expressed as:
\begin{align}
  \label{eq:sgd_weight_update}
  W^{(t+1)} - W^{(t)} & = - \eta \nabla_W \mathcal{L}\left(f_{post}\left(W^{(t)} x+b^{(t)}\right)\right), \,
  b^{(t+1)} - b^{(t)} = - \eta \nabla_b \mathcal{L}\left(f_{post}\left(W^{(t)} x+b^{(t)}\right)\right).
\end{align}
Where $\eta$ is the learning rate, and $\nabla_W \mathcal{L}$ and $\nabla_b \mathcal{L}$ are the gradients of the loss function with respect to the weights and biases, respectively.
In Eq.~(\ref{eq:sgd_weight_update}), we omit the pre-linear layers $f_{pre}$ and use $x$ as the input to the linear layer $f_{linear}$ for simplicity.
Let $y = W x + b$ be the output of the linear layer, $\mathcal{L}' = \mathcal{L} \circ f_{post}$ be the composed loss function.
Starting from the SGD update rule for the weights, we have:
\begin{align}
  W^{(t+1)} - W^{(t)} & = - \eta \nabla_W \mathcal{L}'\left(W^{(t)} x+b^{(t)}\right)                 \\
                      & = - \eta \nabla_y \mathcal{L}' \cdot \nabla_W \left(W^{(t)} x+b^{(t)}\right) \\
                      & = - \eta \nabla_y \mathcal{L}' \cdot x^T,
\end{align}
In practice, we typically use mini-batch SGD, where we average the gradients over a batch of samples. We can represent this as an expectation:
\begin{align}
  W^{(t+1)} - W^{(t)} = - \eta \mathbb{E}_{x \sim p(x)} [\nabla_y \mathcal{L} \cdot x^T]
\end{align}
Given this gradient update rule, we can analyze how it relates to our choice of routing logits in SMILE. Recall that we defined the routing logits as:
\begin{equation}
  r^{(i)} = \left\| V_{k_{gate}}^{(i)^\top} x \right\|_2
\end{equation}
Let's consider the expected weight update over many iterations when $W_{ft}$ is close to $W$:
\begin{align}
  \mathbb{E}[\Delta W] & \approx -\eta \mathbb{E}_{x \sim p(x)} [\nabla_y \mathcal{L} \cdot x^T] \\
                       & = -\eta \mathbb{E}_{x \sim p(x)} [g \cdot x^T]
\end{align}
where $g = \nabla_y \mathcal{L}$ is the gradient of the loss with respect to the output of the linear layer.
Now, let's consider the singular value decomposition (SVD) of the expected weight update:
\begin{equation}
  \mathbb{E}[\Delta W] = U\Sigma V^T
\end{equation}
The right singular vectors $V$ represent the directions in the input space that are most important for the weight updates.
Our routing logits $r^{(i)}$ are based on projecting the input onto the top $k_{gate}$ right singular vectors of the fine-tuned weight difference $\Delta W^{(i)}$. This choice is justified because:
the right singular vectors of $\Delta W^{(i)}$ are likely to be similar to those of $\mathbb{E}[\Delta W]$, as both represent important directions for task-specific updates. In addition, by projecting onto these vectors, we're measuring how much an input aligns with the directions that were most important during fine-tuning for a specific task.

\textbf{A NTK perspective.}
On the other hand, we can analyze the fine-tuning process from a neural tangent kernel (NTK) perspective.
Following~\citep{tangParameterEfficientMultitask2024}, and according to the linear property of the linear layer, we have:
\begin{align}
  f_{linear} & \left(x; \phi_{ft}\right) - f_{linear}\left(x; \phi\right)                                                                                                                                                                  \\
             & = \nabla_{\phi} f_{linear}\left(x; \phi\right)^\top \left(\phi_{ft} - \phi\right)                                                                                                                                           \\
             & \approx - \eta \mathbb{E}_{x'\sim p(x)} \left[\nabla_{\phi} f_{linear}\left(x; \phi\right)^\top \nabla_{\phi} \mathcal{L}'\left(f_{linear}\left(x'; \phi\right)\right)\right]                                               \\
             & = - \eta \mathbb{E}_{x'\sim p(x)} \left[\nabla_{\phi} f_{linear}\left(x; \phi\right)^\top \nabla_{\phi} f_{linear}\left(x'; \phi\right) \nabla_{f_{linear}} \mathcal{L}'\left(f_{linear}\left(x'; \phi\right)\right)\right] \\
             & = - \eta \mathbb{E}_{x'\sim p(x)} \left[ \mathbf{K}(x, x'; \phi) \nabla_{f_{linear}} \mathcal{L}'\left(f_{linear}\left(x'; \phi\right)\right) \right],
\end{align}
Where $\phi$ denotes the pre-trained parameters of the $f_{linear}$, i.e. $W$ and $b$. $\phi_{ft}$ denotes the fine-tuned parameters $W_{ft}$ and $b_{ft}$.
$\mathbf{K}(x, x'; \phi) = \langle \nabla_{\phi} f_{linear}\left(x; \phi\right), \nabla_{\phi} f_{linear}\left(x'; \phi\right) \rangle$ is the neural tangent kernel (NTK)~\citep{jacotNeuralTangentKernel2018} of the linear layer $f_{linear}$, and $\mathcal{L}' = \mathcal{L} \circ f_{post}$ is the composed loss function.
Note that for given $x$, $\mathbf{K}(x, x'; \phi)$ is a constant matrix.

\section{Fine-Tuned Model Performance}
\label{sec:individuals}

\begin{table}
  \centering
  \caption{Performance of fine-tuned CLIP-ViT-B/32 models on eight downstream tasks.}
  \label{tab:individuals-clip-vit-b-32}
  \begin{tabular}{lcccccccc}
    \toprule
    \textbf{Model} & \textbf{SUN397} & \textbf{Cars} & \textbf{RESISC45} & \textbf{EuroSAT} & \textbf{SVHN} & \textbf{GTSRB} & \textbf{MNIST} & \textbf{DTD}  \\
    \midrule
    Pre-trained    & 63.2            & 59.8          & 60.7              & 46.0             & 31.6          & 32.5           & 48.3           & 43.9          \\
    SUN397         & \textbf{75.0}   & 47.0          & 54.3              & 46.5             & 28.3          & 26.4           & 44.3           & 41.6          \\
    Cars           & 56.6            & \textbf{78.3} & 50.9              & 38.4             & 30.2          & 30.6           & 49.7           & 41.8          \\
    RESISC45       & 52.0            & 47.2          & \textbf{95.2}     & 56.9             & 23.9          & 24.3           & 39.7           & 35.9          \\
    EuroSAT        & 49.0            & 39.9          & 33.5              & \textbf{99.0}    & 11.8          & 22.9           & 33.8           & 35.5          \\
    SVHN           & 40.5            & 36.3          & 18.9              & 9.8              & \textbf{97.3} & 27.3           & 81.8           & 23.2          \\
    GTSRB          & 36.8            & 33.0          & 20.6              & 21.3             & 41.2          & \textbf{98.9}  & 30.9           & 23.9          \\
    MNIST          & 50.3            & 40.0          & 31.3              & 17.7             & 50.1          & 19.3           & \textbf{99.6}  & 30.7          \\
    DTD            & 54.6            & 51.3          & 36.9              & 25.0             & 28.9          & 21.8           & 47.3           & \textbf{79.7} \\
    \bottomrule
  \end{tabular}
\end{table}

\begin{table}
  \centering
  \caption{Performance of fine-tuned CLIP-ViT-L/14 models on eight downstream tasks.}
  \label{tab:individuals-clip-vit-l-14}
  \begin{tabular}{lcccccccc}
    \toprule
    \textbf{Model} & \textbf{SUN397} & \textbf{Cars} & \textbf{RESISC45} & \textbf{EuroSAT} & \textbf{SVHN} & \textbf{GTSRB} & \textbf{MNIST} & \textbf{DTD}  \\
    \midrule
    Pre-trained    & 68.3            & 77.8          & 71.0              & 58.9             & 58.4          & 50.6           & 76.4           & 55.5          \\
    SUN397         & \textbf{82.8}   & 68.4          & 58.1              & 49.9             & 55.0          & 46.3           & 79.5           & 52.8          \\
    Cars           & 67.8            & \textbf{92.9} & 68.7              & 56.4             & 51.7          & 47.7           & 80.5           & 55.6          \\
    RESISC45       & 65.6            & 69.0          & \textbf{97.4}     & 64.3             & 38.3          & 46.6           & 77.7           & 49.9          \\
    EuroSAT        & 65.2            & 69.0          & 40.6              & \textbf{99.2}    & 33.4          & 45.6           & 73.5           & 47.1          \\
    SVHN           & 66.4            & 69.0          & 54.0              & 19.7             & \textbf{97.9} & 48.7           & 92.2           & 50.1          \\
    GTSRB          & 63.4            & 64.8          & 38.7              & 19.6             & 71.0          & \textbf{99.2}  & 75.1           & 45.8          \\
    MNIST          & 56.0            & 49.8          & 53.5              & 26.6             & 48.2          & 33.1           & \textbf{99.8}  & 47.1          \\
    DTD            & 66.8            & 75.3          & 65.5              & 43.7             & 49.5          & 45.0           & 68.5           & \textbf{85.5} \\
    \bottomrule
  \end{tabular}
\end{table}

\begin{table}
  \centering
  \caption{Performance of full fine-tuned Flan-T5-Base models on eight downstream tasks.}
  \label{tab:individuals-flan-t5-base}
  \begin{tabular}{lcccccccc}
    \toprule
    \textbf{Model} & \textbf{CoLA} & \textbf{MNLI} & \textbf{MRPC} & \textbf{QNLI} & \textbf{QQP}  & \textbf{RTE}  & \textbf{SST2} & \textbf{STSB} \\
    \midrule
    Pre-trained    & 69.1          & 56.5          & 76.2          & 88.4          & 82.1          & 80.1          & 91.2          & 62.2          \\
    CoLA           & \textbf{75.0} & 37.2          & 72.8          & 87.6          & 80.4          & 76.9          & 91.4          & 63.6          \\
    MNLI           & 65.9          & \textbf{83.4} & 75.7          & 89.2          & 82.6          & 78.0          & 90.6          & 66.2          \\
    MRPC           & 63.4          & 48.3          & \textbf{87.5} & 85.8          & 81.1          & 72.6          & 88.1          & 76.1          \\
    QNLI           & 68.7          & 39.2          & 75.5          & \textbf{91.5} & 81.3          & 78.3          & 91.6          & 68.2          \\
    QQP            & 59.1          & 50.4          & 73.8          & 88.3          & \textbf{85.4} & 81.2          & 90.8          & 75.9          \\
    RTE            & 65.4          & 51.1          & 69.6          & 88.7          & 80.8          & \textbf{85.9} & 90.3          & 68.9          \\
    SST2           & 67.8          & 54.0          & 76.5          & 87.8          & 83.4          & 80.5          & \textbf{93.6} & 63.6          \\
    STSB           & 69.3          & 49.3          & 76.5          & 89.0          & 81.7          & 77.6          & 90.1          & \textbf{88.7} \\
    \bottomrule
  \end{tabular}
\end{table}

\begin{table}
  \centering
  \caption{Performance of LoRA fine-tuned ($r_{LoRA}=16$) Flan-T5-Base models on eight downstream tasks.}
  \label{tab:individuals-flan-t5-base_lora16}
  \begin{tabular}{lcccccccc}
    \toprule
    \textbf{Model} & \textbf{CoLA} & \textbf{MNLI} & \textbf{MRPC} & \textbf{QNLI} & \textbf{QQP}  & \textbf{RTE}  & \textbf{SST2} & \textbf{STSB} \\
    \midrule
    Pre-trained    & 69.1          & 56.5          & 76.2          & 88.4          & 82.1          & 80.1          & 91.2          & 62.2          \\
    CoLA           & 69.1          & 39.9          & 75.2          & 89.1          & 81.1          & 81.9          & 90.7          & 54.0          \\
    MNLI           & \textbf{69.4} & \textbf{82.7} & 73.8          & 89.3          & 82.0          & 79.4          & 90.9          & 68.1          \\
    MRPC           & 64.0          & 44.9          & \textbf{85.5} & 82.6          & 81.0          & 69.0          & 88.6          & 73.6          \\
    QNLI           & 68.9          & 52.7          & 76.7          & \textbf{90.9} & 82.8          & 79.8          & 91.5          & 68.9          \\
    QQP            & 65.0          & 54.6          & 75.7          & 89.0          & \textbf{84.0} & 81.6          & 90.7          & 75.3          \\
    RTE            & 64.9          & 51.8          & 69.4          & 89.2          & 79.8          & \textbf{84.5} & 90.6          & 70.1          \\
    SST2           & 68.3          & 56.6          & 76.0          & 88.5          & 83.4          & 79.8          & \textbf{92.9} & 62.6          \\
    STSB           & 65.7          & 1.7           & 67.4          & 89.3          & 80.1          & 79.8          & 90.8          & \textbf{87.4} \\
    \bottomrule
  \end{tabular}
\end{table}

In this section, we present the performance of the fine-tuned models on their corresponding test sets.
These results serve as a baseline for evaluating the effectiveness of our proposed model fusion technique.

Tables \ref{tab:individuals-clip-vit-b-32} and \ref{tab:individuals-clip-vit-l-14} show the performance of fine-tuned CLIP-ViT-B/32 and CLIP-ViT-L/14 models, respectively, on eight image classification tasks.
These tasks include SUN397~\citep{xiaoSUNDatabaseLargescale2010}, Cars~\citep{krause3DObjectRepresentations2013}, RESISC45~\citep{chengRemoteSensingImage2017}, EuroSAT~\citep{helberEuroSATNovelDataset2019}, SVHN~\citep{netzerReadingDigitsNatural2011}, GTSRB~\citep{stallkampManVsComputer2012}, MNIST~\citep{lecunGradientbasedLearningApplied1998}, and DTD~\citep{cimpoiDescribingTexturesWild2014}.
For image classification tasks, we report the classification accuracy.
Tables \ref{tab:individuals-flan-t5-base} and \ref{tab:individuals-flan-t5-base_lora16} present the performance of Flan-T5-Base models on eight text generation tasks from GLUE benchmark~\citep{wangGLUEMultiTaskBenchmark2018}, using full fine-tuning and LoRA fine-tuning ($r_{LoRA}=16$) respectively.
We report Spearman’s $\rho$ for STSB and exact match accuracy for other tasks.
In particular for STSB, if the text outputs can not be parsed into a valid float number, we assign a score of zero.
The datasets and fine-tuned models are accessible to the public on HuggingFace~\citep{wolfHuggingFaceTransformersStateoftheart2020}. For further information, please consult FusionBench~\citep{tangFusionBenchComprehensiveBenchmark2024}.

Several key observations can be made from these results:
\begin{enumerate}
  \item \textit{Task-Specific Improvements}:
        Across all model types, fine-tuning consistently improves performance on the target task compared to the pre-trained model. This demonstrates the effectiveness of task-specific adaptation.
  \item \textit{Negative Transfer}:
        In some cases, we observe negative transfer, where fine-tuning on one task harms performance on another task. For example, in Table \ref{tab:individuals-clip-vit-b-32}, the SVHN-tuned model performs worse on EuroSAT (9.8\%) compared to the pre-trained model (46.0\%).
  \item \textit{Task Relatedness}:
        Some fine-tuned models show improved performance on related tasks.
        For example, in Table \ref{tab:individuals-flan-t5-base}, the MNLI-tuned model performs well on QNLI, suggesting a transfer of relevant knowledge between these natural language inference tasks.
  \item \textit{Varying Task Difficulty}:
        The diagonal entries reveal that some tasks are inherently more challenging than others. For instance, in the CLIP-ViT models, EuroSAT and MNIST consistently achieve very high accuracy, while SUN397, Cars and DTD prove more challenging.
  \item \textit{Model Size Impact}:
        Comparing CLIP-ViT-B/32 and CLIP-ViT-L/14 results, we generally see improved performance with the larger model, indicating that model capacity plays a role in both task-specific performance and generalization.
\end{enumerate}

\section{Hyperparameter Analysis}
\label{sec:hyperparameter-analysis}

\begin{figure}
  \begin{subfigure}{0.44\textwidth}
    \centering
    \includegraphics[width=\linewidth]{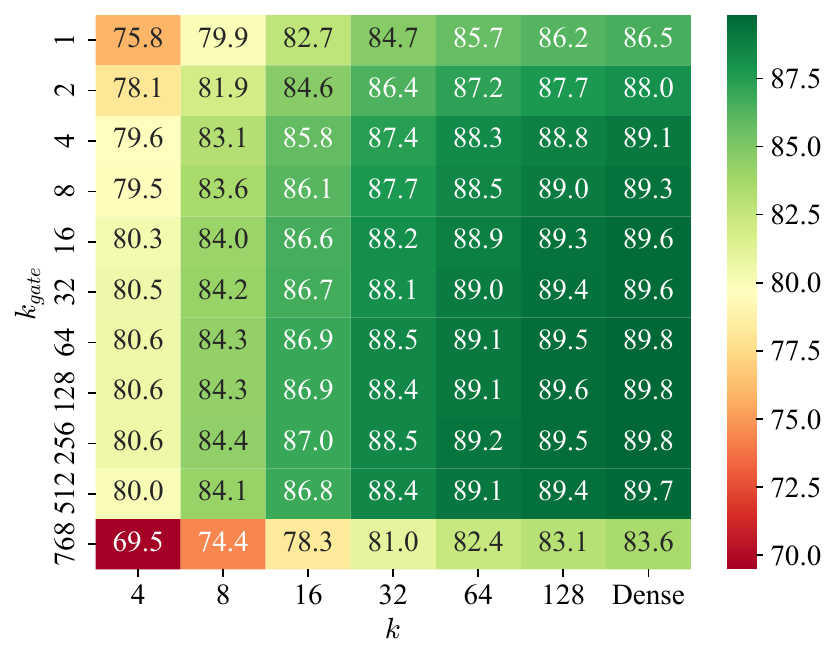}
    \caption{Heatmap of average accuracy.}
    \label{fig:clip-vit-b-32_hp-acc}
  \end{subfigure}%
  \begin{subfigure}{0.55\textwidth}
    \centering
    \includegraphics[width=\linewidth]{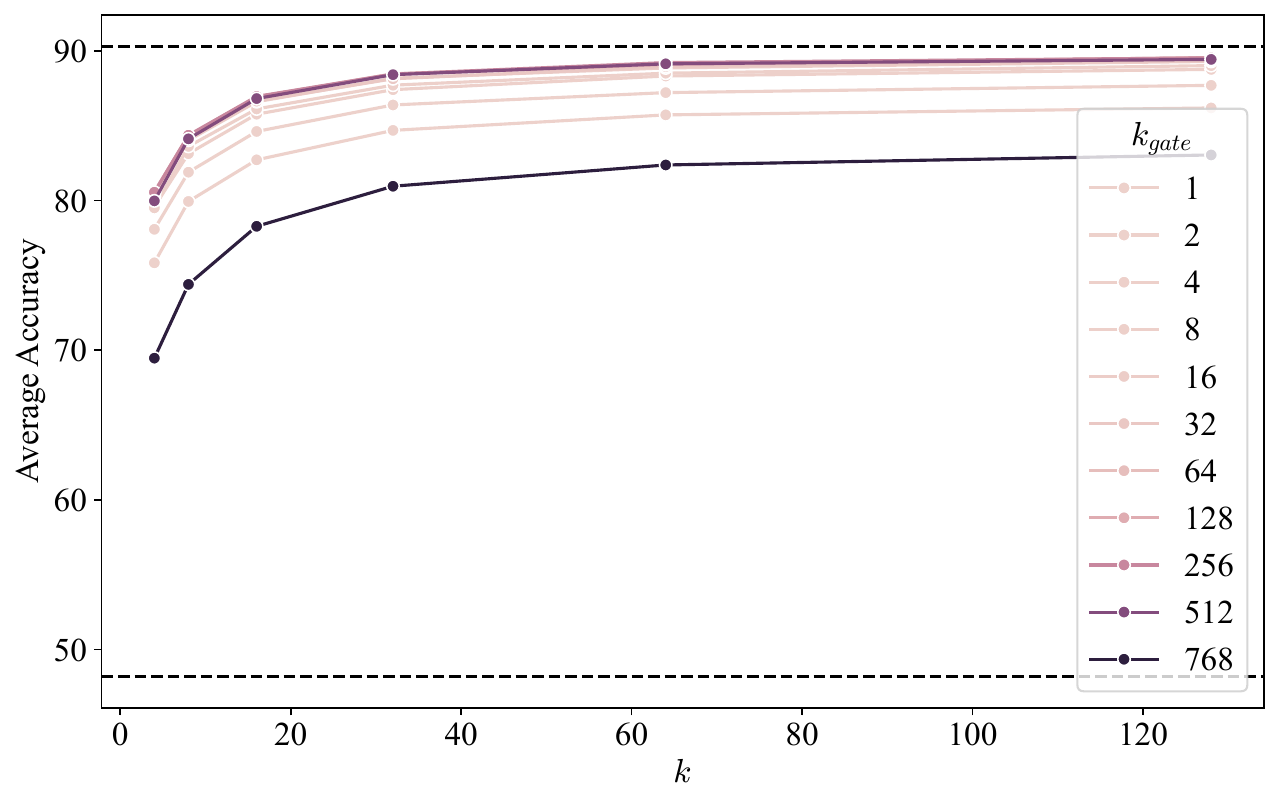}
    \caption{Line plot of average accuracy.}
    \label{fig:clip-vit-b-32_line-acc}
  \end{subfigure} \\
  \begin{subfigure}{0.44\textwidth}
    \centering
    \includegraphics[width=\linewidth]{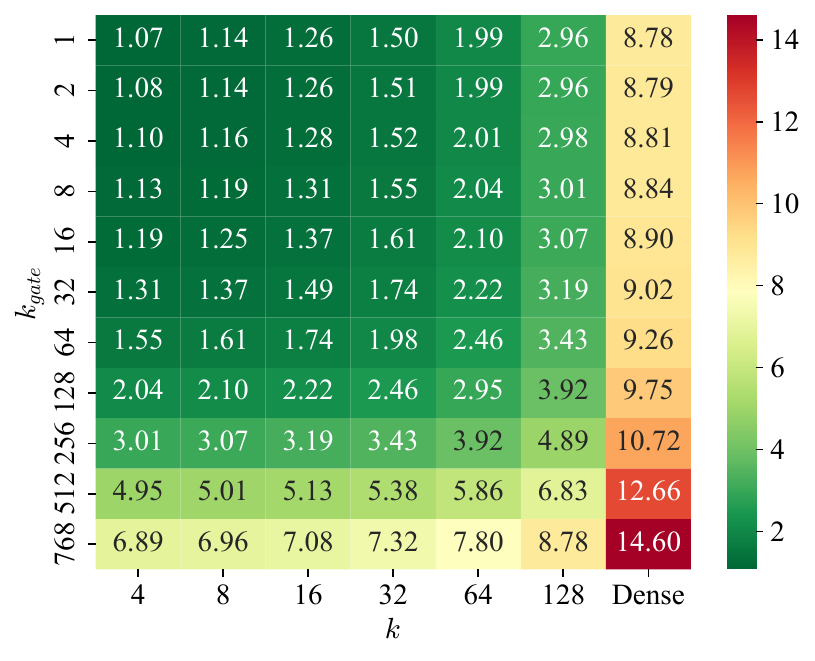}
    \caption{Heatmap of the normalized number of parameters.}
    \label{fig:clip-vit-b-32_hp-params}
  \end{subfigure}%
  \begin{subfigure}{0.55\textwidth}
    \centering
    \includegraphics[width=\linewidth]{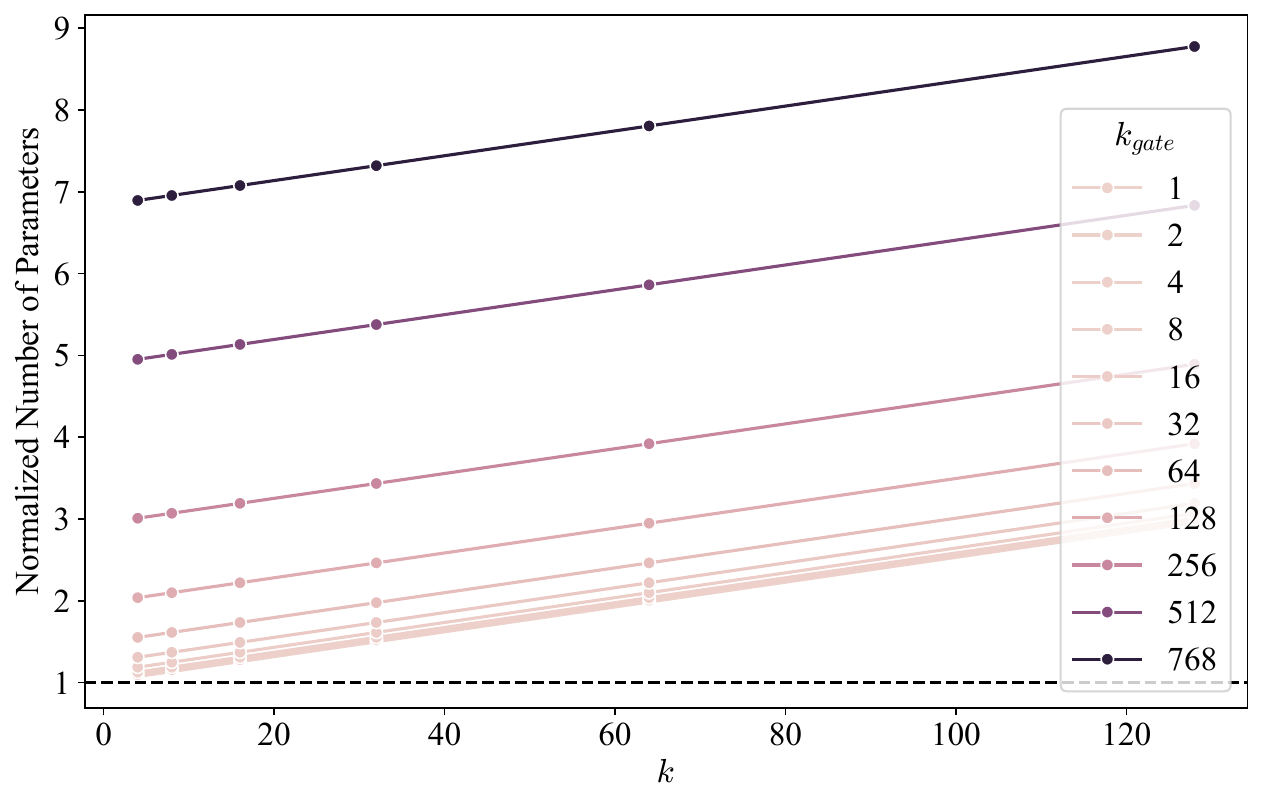}
    \caption{Line plot of normalized number of parameters.}
    \label{fig:clip-vit-b-32_line-params}
  \end{subfigure}
  \caption{Hyperparameter analysis of the CLIP-ViT-B/32 model on eight image classification datasets.
    Here we show how different values of hyperparameters $k$ and $k_{gate}$ affect the average performance and the number of parameters (normalized by the number of parameters in the original model) in the upscaled model.
    We also show the average accuracy of pre-trained models and individual fine-tuned models in subfigure (b).}
  \label{fig:clip-vit-b-32_hp}
\end{figure}

\begin{figure}
  \begin{subfigure}{0.43\textwidth}
    \centering
    \includegraphics[width=\linewidth]{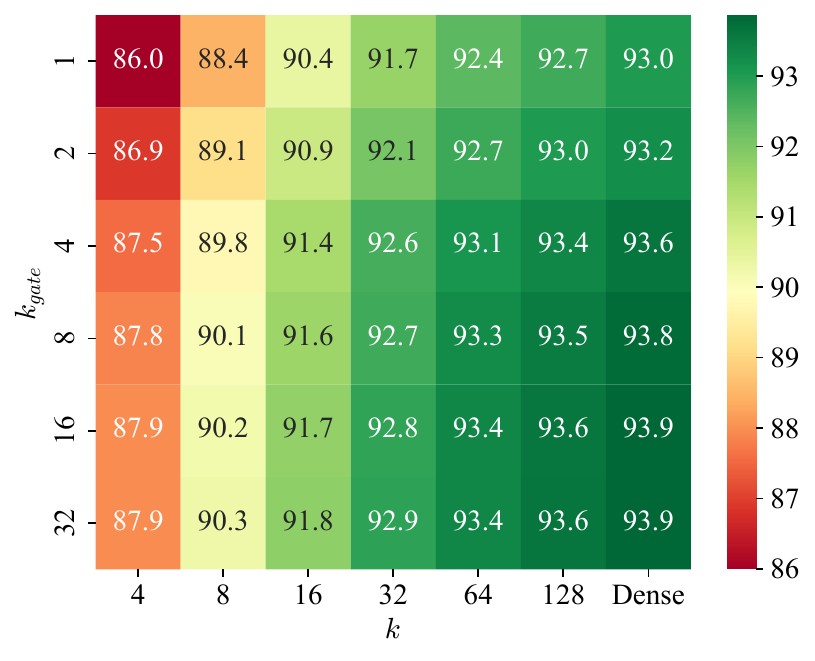}
    \caption{Heatmap of average performance.}
    \label{fig:clip-vit-l-14_hp-acc}
  \end{subfigure}%
  \hfill
  \begin{subfigure}{0.55\textwidth}
    \centering
    \includegraphics[width=\linewidth]{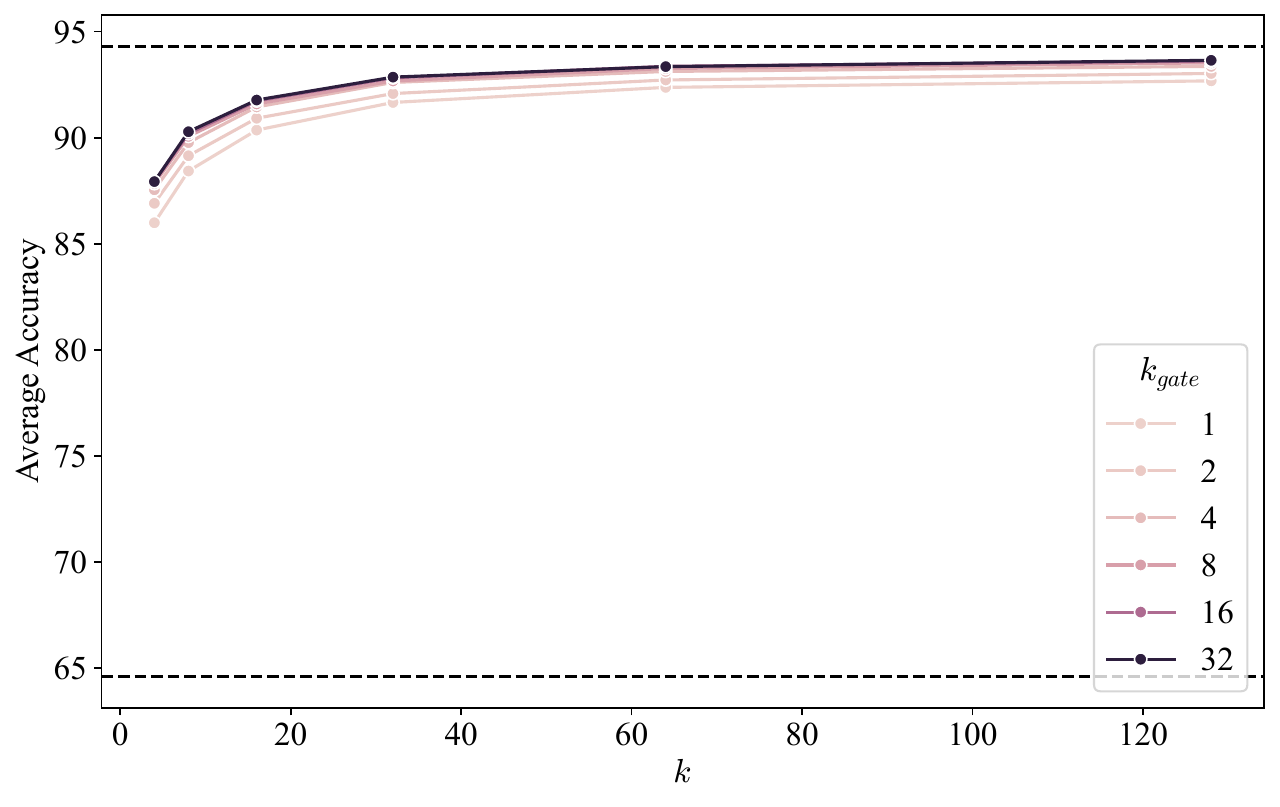}
    \caption{Line plot of average performance.}
    \label{fig:clip-vit-l-14_line-acc}
  \end{subfigure} \\
  \begin{subfigure}{0.42\textwidth}
    \centering
    \includegraphics[width=\linewidth]{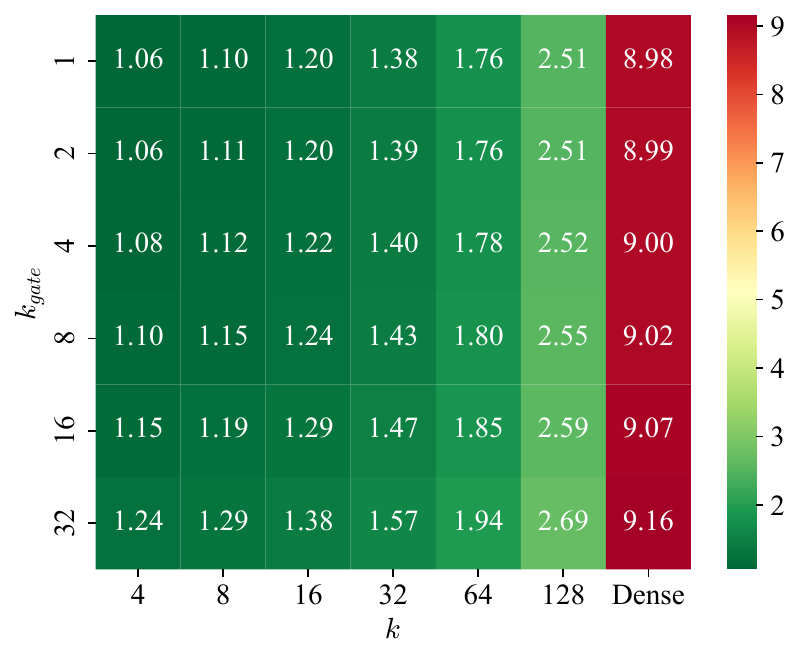}
    \caption{Heatmap of normalized parameter count.}
    \label{fig:clip-vit-l-14_hp-params}
  \end{subfigure}%
  \hfill
  \begin{subfigure}{0.55\textwidth}
    \centering
    \includegraphics[width=\linewidth]{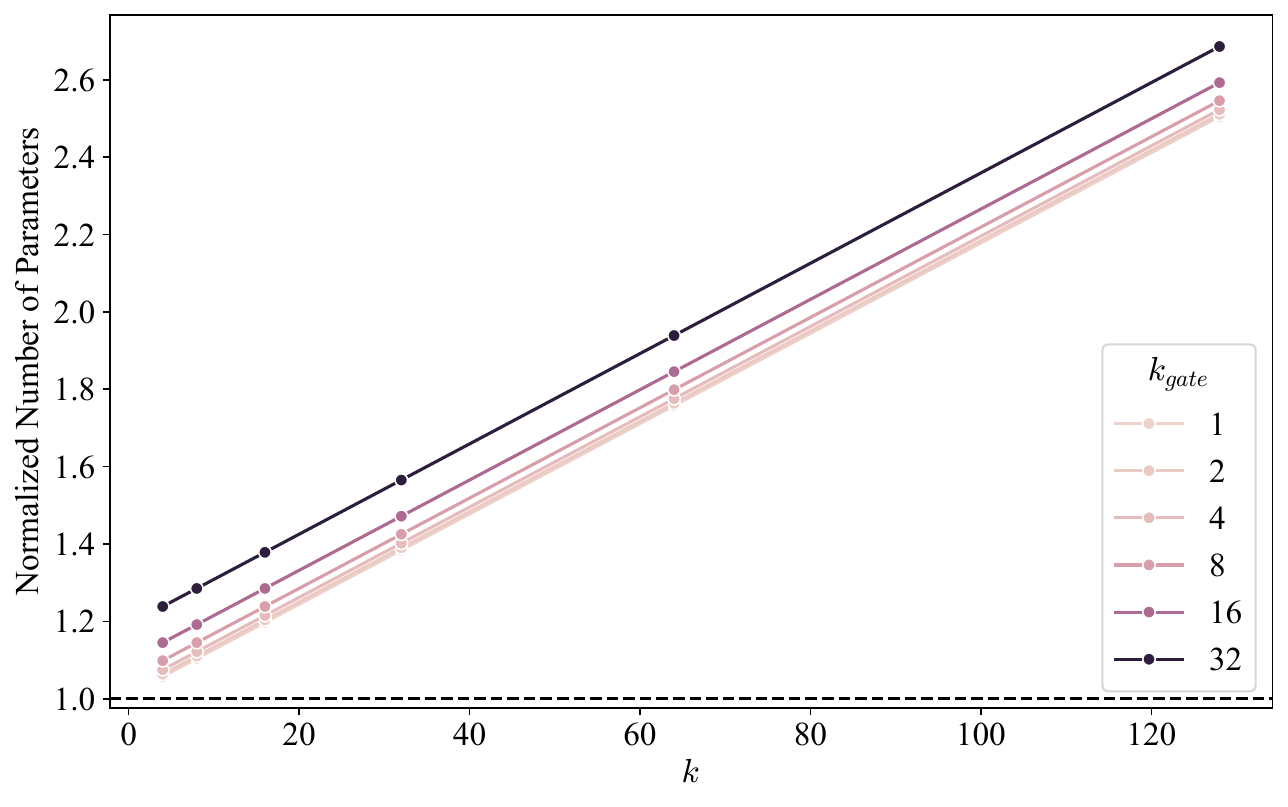}
    \caption{Line plot of normalized parameter count.}
    \label{fig:clip-vit-l-14_line-params}
  \end{subfigure}
  \caption{Hyperparameter analysis of the CLIP-ViT-L/14 model on eight image classification datasets.
    Here we show how different values of hyperparameters $k$ and $k_{gate}$ affect the average performance and the number of parameters (normalized by the number of parameters in the original model) in the upscaled model.
    We also show the average accuracy of pre-trained models and individual fine-tuned models in subfigure (b).}
  \label{fig:clip-vit-l-14_hp}
\end{figure}

In this section, we present a comprehensive analysis of the hyperparameters $k$ and $k_{gate}$ for the CLIP-ViT-B/32 and CLIP-ViT-L/14 models across eight image classification datasets. We examine their impact on both model performance (average accuracy) and model complexity (number of parameters).
We also test the extreme cases when $k = \infty$, which corresponds to full-rank experts, denoted as ``Dense'' in the figures.
We normalize the number of parameters by the number of parameters in the original model (87.5M for CLIP-ViT-B/32 and 303M for CLIP-ViT-L/14) to facilitate comparison.

\textbf{CLIP Models.}
From Figures~\ref{fig:clip-vit-b-32_hp} and~\ref{fig:clip-vit-l-14_hp}, we observe that
(1) The performance of the upscaled models is generally better than the pre-trained models, which demonstrates the effectiveness of our fine-tuning strategy.
(2) increasing the values of \( k \)  generally improves the performance of both the CLIP-ViT-B/32 and CLIP-ViT-L/14 models, though at the cost of increased model complexity.
(3) Increased values of \( k_{gate} \) improve the performance of the upscaled models at the beginning, but the performance starts to decrease when \( k_{gate} \) is too large. This observation is consistent with our discussion in Section~\ref{sec:sparse_mixture_of_experts} that a larger \( k_{gate} \) may result in a less discriminative gating mechanism.
(4) Better performance preservation can be achieved with the CLIP-ViT-L/14 model than with the CLIP-ViT-B/32 model, which is consistent with our discussion in Section~\ref{sec:parameter_interference} that the larger model has more dimension redundancy and is less severe to the parameter interference problem.

In practice, when selecting hyperparameters for the upscaled models, it is crucial to balance the trade-off between performance and parameter overhead.
Take CLIP-ViT-B/32 as an example, a good trade-off between performance and parameter overhead can be achieved with \( k \approx 32 \) and \( k_{gate} \approx 16 \).
For the CLIP-ViT-L/14 model, $ k \approx 64$ and $ k_{gate} \approx 8$ are recommended.
By doing so, we obtain a multi-task model that achieves around 98\% of the performance of the fine-tuned model with only 20\% of the total parameters compared to maintaining eight individual fine-tuned models for each task.
Note that the upscaled SMILE model is sparsely inferenced, increasing the number of parameters by $N$ only increases the activated parameters by about $N/T$ for each token.
Even with an extreme focus on storage and inference costs, only 7\% of the parameters overhead can achieve an average multi-task performance of about 90\% of the individual fine-tuned models.

\textbf{Flan-T5 Models.}
Figure~\ref{fig:flan-t5-base_hp} shows the hyperparameter analysis of the Flan-T5-Base models on eight tasks from the GLUE benchmark.
We conduct the same analysis for both full fine-tuned models and LoRA fine-tuned models with $r_{LoRA}=16$.
It is observed that the performance is relatively stable, with most configurations yielding accuracies around 85.1 to 85.6 for the full fine-tuned models and around 83.8 to 84.0 for the LoRA fine-tuned models.
Upscaling LoRA fine-tuned models is very parameter-efficient, with the number of parameters increasing by only 2\% to 7\% compared to the original dense model.
For a balanced trade-off between performance and parameter overhead, consider setting $k \approx 32$ and $k_{gate} \approx 8$ for the full fine-tuned model fusion, and $k \approx 8$ and $k_{gate} \approx 2$ for the LoRA fine-tuned model fusion.

\section{Large-Scale Model Experiments}
\label{appendix:large-scale-experiments}

This appendix provides additional details and results for our experiments with large-scale models, specifically the \texttt{Mistral-7B} series.
We used the following models in our experiments, which are available on the HuggingFace:
\begin{itemize}
  \item Base pre-trained model ($M_0$):
        \href{https://huggingface.co/mistralai/Mistral-7B-v0.1}{\texttt{mistralai/Mistral-7B-v0.1}}
  \item Expert model $M_1$:
        \href{https://huggingface.co/meta-math/MetaMath-Mistral-7B}{\texttt{meta-math/MetaMath-Mistral-7B}}
  \item Expert model $M_2$:
        \href{https://huggingface.co/cognitivecomputations/dolphin-2.1-mistral-7b}{\texttt{cognitivecomputations/dolphin-2.1-mistral-7b}}
  \item Expert model $M_3$:
        \href{https://huggingface.co/uukuguy/speechless-code-mistral-7b-v1.0}{\texttt{uukuguy/speechless-code-mistral-7b-v1.0}}
\end{itemize}
\begin{table}
  \centering
  \caption{
    Detailed performance comparison of individual models and various SMILE models with different \( k \) values. For all upscaled models, the \( k_{gate} \) value was set to 8.
  }
  \label{table:appendix-mistral-7b-full}
  \begin{tabular}{lcccc}
    \toprule
    \textbf{Model}                       & \textbf{MMLU}  & \textbf{TruthfulQA} & \textbf{GSM8K} & \textbf{ARC Challenge} \\
    \midrule
    $M_0$ (pre-trained)                  & 59.64          & 42.62               & 38.81          & 53.92                  \\
    $M_1$                                & 60.56          & 44.79               & \textbf{71.49} & 51.02                  \\
    $M_2$                                & 60.56          & \textbf{55.88}      & 56.93          & 57.00                  \\
    $M_3$                                & \textbf{61.18} & 47.47               & 48.98          & \textbf{57.68}         \\
    \midrule
    $M_{0;123}$ $(7.3 \text{B}, k=8)$    & 60.28          & 46.31               & 46.55          & \textbf{55.55}         \\
    $M_{0;123}$ $(7.5 \text{B}, k=32)$   & 60.37          & 49.49               & 55.04          & 54.52                  \\
    $M_{0;123}$ $(8.3 \text{B}, k=128)$  & 60.43          & 50.91               & 63.76          & 54.35                  \\
    $M_{0;123}$ $(9.3 \text{B}, k=256)$  & 60.53          & 51.83               & 65.58          & 54.01                  \\
    $M_{0;123}$ $(11.2 \text{B}, k=512)$ & \textbf{60.66} & \textbf{52.79}      & \textbf{67.85} & 54.35                  \\    \bottomrule
  \end{tabular}
\end{table}

For the SMILE models, the hyperparameter settings were as follows: \( k_{gate} \) was consistently set to 8 across all experiments, while \( k \) ranged from 8 to 512 (including 8, 16, 32, 64, 128, 256, 384, and 512), as shown in Figure \ref{fig:mistral-7b_gsm8k}.
Table \ref{table:appendix-mistral-7b-full} provides a more comprehensive view of the performance of individual models and various SMILE models with different \( k \) values.
We use \href{https://github.com/EleutherAI/lm-evaluation-harness}{\texttt{EleutherAI/lm-evaluation-harness}}~\citep{eval-harness} to evaluate the models on the four tasks: MMLU, TruthfulQA, GSM8K, and ARC Challenge.
We merge the models on host memory and evaluate them on two NVIDIA 4090 GPUs with 24GB of memory each.

It is notable that as the value of $k$ increases, we generally see improved performance, especially in tasks like GSM8K and TruthfulQA.
The results also show a clear trade-off between model size and performance. The SMILE model with $k=8$ (7.3B parameters) already achieves comparable or better results than the pre-trained model on all tasks, while larger models ($k=512$, 11.2B parameters) approach the performance of individual expert models.

\textbf{Limitations and future work for LLMs.}
Here we provide a brief discussion of the limitations of our experiments and potential directions for future work.
\begin{itemize}
  \item \textbf{Limited expert model pool.}
        In the experiments for CLIP models and Flan-t5 models, we use eight expert models to evaluate the performance of the SMILE model. However, in the Mistral-7B experiments, the experiments are currently limited to three expert models, which may not fully demonstrate the method's capabilities with a larger, more diverse set of experts. Future work could explore the impact of additional expert models on the performance of the SMILE model.
  \item \textbf{LoRA fine-tuning.}
        In the experiments, we use full fine-tuned Mistral-7B models as expert models. Where the linear models are upscaled into MoE modules and the remaining parts of the model are copied directly from the pre-trained model.
        The reason for this is that the top Mistral-7B models available on HuggingFace are now fully fine-tuned.
        This approach, however, may not fully exploit SMILE's potential.
        A more effective strategy could involve using LoRA fine-tuned models as expert models. In this scenario, only specific linear layers would be fine-tuned using low-rank techniques, with the rest of the model remaining frozen. This approach could potentially enhance SMILE's efficiency and effectiveness.
        As we have shown in the Flan-T5 experiments, LoRA fine-tuning can significantly reduce the number of additional parameters required to achieve comparable performance.
\end{itemize}

\end{document}